\documentclass[twoside]{article}

\usepackage[accepted]{aistats2026}

\usepackage{hyperref}
\usepackage{url}
\usepackage[T1]{fontenc}

\usepackage{amsmath}
\usepackage{amsthm}
\usepackage{graphicx}      %

\usepackage{subcaption}
\usepackage{algorithm}
\usepackage[noend]{algorithmic}

\usepackage{wrapfig}
\usepackage{transparent}
\usepackage{booktabs}
\usepackage{multirow}
\usepackage{makecell}
\usepackage{xspace}
\usepackage{xcolor}
\usepackage[framemethod=tikz]{mdframed}
\usepackage{enumitem}
\usepackage{mathtools}
\usepackage{todonotes}
\usepackage{tabularx}
\usepackage{listings}
\usepackage{array}

\usepackage{amsmath,amsfonts}

\def\1{\bm{1}}

\DeclareMathAlphabet{\mathsfit}{\encodingdefault}{\sfdefault}{m}{sl}
\SetMathAlphabet{\mathsfit}{bold}{\encodingdefault}{\sfdefault}{bx}{n}

\newcommand{\E}{\mathbb{E}}

\DeclareMathOperator{\TV}{\mathrm{TV}}

\newcommand{\PP}{\mathbb{P}}
\newcommand{\mint}[1]{\min\left\{1, \frac{p(x_{#1} \mid x_{0:#1-1})}{q(x_{#1} \mid x_{0:#1-1})}\right\}}

\usepackage{tabularx}
\usepackage{makecell}

\usepackage[framemethod=tikz]{mdframed}
\newmdenv[
  backgroundcolor=gray!10,  %
  linecolor=black,          %
  linewidth=1pt,            %
  roundcorner=5pt,          %
  skipabove=10pt,
  skipbelow=10pt,
  innerleftmargin=10pt,
  innerrightmargin=10pt,
  innertopmargin=10pt,
  innerbottommargin=10pt
]{mybox}

\usepackage[many]{tcolorbox}

\newcommand{\SampleBoxText}[3][]{%
  \begin{tcolorbox}[
    enhanced, breakable,
    colframe=black!75!white, colback=gray!10!white,
    title={#2}
  ]
   
#3

  \end{tcolorbox}%
}

\newcommand{\codecell}[1]{%
  \begingroup
  \ttfamily\obeyspaces\obeylines
  \detokenize{#1}%
  \par
  \endgroup
}

\newtheorem{theorem}{Theorem}
\newtheorem{lemma}{Lemma}
\newtheorem{proposition}{Proposition}
\newtheorem{definition}{Definition}

\newcommand{\nj}[1]{{#1}}

\newcommand{\method}{\textsc{Diversed}\xspace}
\definecolor{ForestGreen}{HTML}{228B22}

\usepackage[round]{natbib}

\begin{document}

\runningtitle{DIVERSED: Relaxed Speculative Decoding  via Dynamic Ensemble Verification}

\runningauthor{Wang, Kasa, MS, Kasa, Zou, Negi, Zhang, Jiang, Song}

\twocolumn[

\aistatstitle{DIVERSED: Relaxed Speculative Decoding via \\Dynamic Ensemble Verification}

\aistatsauthor{Ziyi Wang$^{*,1}$, Siva Rajesh Kasa$^{*,2}$, Ankith M S$^2$, Santhosh Kumar Kasa$^2$, Jiaru Zou$^3$,\\
  \textbf{Sumit Negi$^2$, Ruqi Zhang$^1$, Nan Jiang$^4$, Qifan Song$^1$}}

\aistatsaddress{$^1$Purdue University, $^2$Amazon Inc., $^3$University of Illinois Urbana-Champaign, $^4$University of Texas at El Paso}
]
\def\thefootnote{*}\footnotetext{Equal contribution.}\def\thefootnote{\arabic{footnote}}

\begin{abstract}
Speculative decoding is an effective technique for accelerating large language model inference by drafting multiple tokens in parallel. In practice, its speedup is often bottlenecked by a rigid verification step that strictly enforces the accepted token distribution to exactly match the target model. This constraint leads to the rejection of many plausible tokens, lowering the acceptance rate and limiting overall time speedup.
To overcome this limitation, we propose \underline{D}ynam\underline{I}c \underline{VE}rification \underline{RE}laxed \underline{S}p\underline{E}culative \underline{D}ecoding (DIVERSED), a relaxed verification framework that improves time efficiency while preserving generation quality. 
DIVERSED: learns an ensemble-based verifier that blends the draft and target model distributions with a task-dependent and context-dependent weight. 
We provide theoretical justification for our approach and demonstrate empirically that DIVERSED achieves substantially higher inference efficiency compared to standard speculative decoding methods. Code is available at: \url{https://github.com/comeusr/diversed}.

\end{abstract}

\section{INTRODUCTION}
Recent advances in large language models have demonstrated strong capabilities on demanding tasks, including mathematical reasoning \citep{kojima2022large, wei2022chain, zou2025latentmas}, complex programming~\citep{roziere2023code, yan2024decoding, jiang2024survey}, and tool use \citep{schick2023toolformer, qin2023toolllm, zou2025tattoo}. 
Yet, these gains are achieved with billion-scale models, which substantially increase inference cost and make efficiency a central bottleneck. 
Existing efforts have improved runtime via faster attention~\citep{dao2022flashattention}, model compression~\citep{dettmers2022gpt3}, and accelerating decoding~\citep{leviathan2023fast, narasimhan2024faster}.

\begin{figure}
    \centering
    \includegraphics[width=1\linewidth]{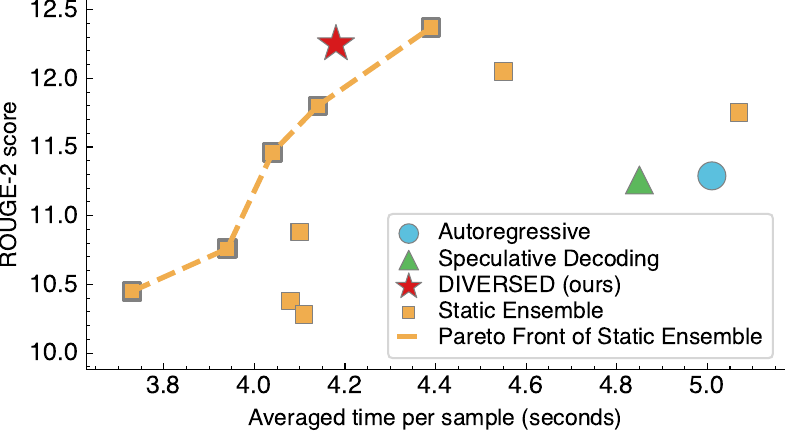} 
    \caption{Our method improves upon the Pareto frontier achieved by static ensemble, attaining better trade-offs between inference time and generation quality.
	Results are reported on the CNNDM dataset using the target/draft model pair: Llama-3.1-8B-Instruct and Llama-3.2-1B-Instruct, with temperature being $0$.  More results are in Appendix Figure~\ref{fig:rouge_time_qwen_gemma}.}
	\label{fig:llama_cnndm_comp}
\end{figure}

\begin{figure*}[!t]
  \centering
  \includegraphics[width=.97\textwidth]{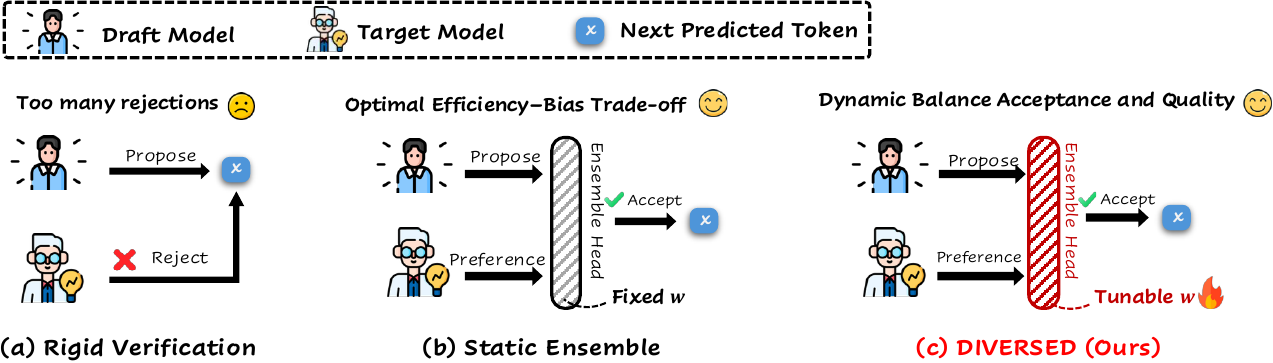}
  \caption{
Compared with classic speculative decoding (shown in (a) rigid verification), our \method (c) achieves a higher acceptance rate, comparable accuracy, and lower wall-clock time.
Compared with static ensemble (b), \method (c) attains higher accuracy via a tunable ensemble weight that adapts to the task and the context.
}
  \label{fig:diversed_schematic}
\end{figure*}

Among decoding strategies, speculative decoding has emerged as a particularly effective approach~\citep{leviathan2023fast, chen2023accelerating}. It employs a small draft model to propose tokens and a large target model to verify them in parallel, reducing latency while maintaining target-level quality.
Subsequent research has progressed along two main directions: (1) lossless methods that preserve exact agreement with the target model, and (2) lossy methods that accept more draft tokens at the cost of potential quality degradation. 
Lossless work includes self-speculative architectures that restructure or internalize the draft process~\citep{Cai2024Medusa, li2024eagle, Elhoushi2024LayerSkip}, and alignment strategies that train draft models to produce more target-like outputs, thereby increasing acceptance~\citep{zhou2023distillspec, goel2024direct, liu2023online}. On the other hand, lossy approaches intentionally relax token-level agreement with the target, trading performance degradation for higher acceptance and throughput~\citep{bachmann2025judge}. Recent theory further formalizes this acceptance rate and generation quality tradeoff~\citep{yin2024theoretical}.

\nj{This work follows the lossy speculative decoding line, aiming to increase acceptance while incurring minimal degradation in generation quality. This is feasible because many tokens rejected by the target verifier can still lead to correct solutions~\citep{yan2024decoding, bachmann2025judge}; rejecting them outright can limit overall time speedup. 
An intuitive idea is to relax the verification rule by interpolating the draft and target distributions using a static weight, noted as static ensemble. It directly traces the Pareto frontier characterized by~\citet{yin2024theoretical}.
Figure~\ref{fig:example} illustrates both a success case and a failure case. While a static ensemble can improve acceptance without harming quality in some instances, it can also degrade generation quality and yield incorrect solutions. This suggests that the ensemble rule should be context-dependent, rather than fixed across all generated contexts.}

To address this limitation, we propose \underline{D}ynam\underline{I}c \underline{VE}rification \underline{RE}laxed \underline{S}p\underline{E}culative \underline{D}ecoding (\method), a dynamic ensemble verifier that mixes the target and draft distributions with context- and token-dependent weights.
These weights are learned by optimizing a regularized objective that maximizes expected task reward on the dataset while encouraging a high acceptance rate.
As shown in Figure~\ref{fig:llama_cnndm_comp}, \method achieves higher acceptance at comparable quality, surpassing the Pareto frontier of static ensemble. A visual comparison of \method against existing baselines is shown in Figure~\ref{fig:diversed_schematic}.

Theoretically, we show that (1) the \emph{static ensemble} verifier lies on the acceptance-quality Pareto front proposed by~\citet{yin2024theoretical}.
Importantly, this implies that practitioners need not maintain a fleet of draft--target pairs for different user requirements: a \emph{single pair of models} can flexibly serve diverse latency--quality trade-offs simply by adjusting the static ensemble weight.
(2) We also derive an exact, step-dependent expression for the expected accepted length. This characterization captures position effects overlooked in prior analyses and eliminates both the \textit{i.i.d.} assumption adopted in earlier work. This extension is technically non-trivial.

Experimentally, we show (1) \method consistently increases acceptance rates, while preserving task accuracy across all model pairs and task datasets.
 (2) The wall-clock speedup closely correlated to acceptance rate: increases in acceptance rate reliably translate into end-to-end latency reductions. (3) The efficiency–quality trade-off exhibits a clear Pareto front, on which \method consistently achieves a strong balance between speed and accuracy.

\section{PRELIMINARY} \label{sec:prelim}
\textbf{Speculative decoding.} In autoregressive language models, a target model generates tokens $(x_1,\dots,x_T)$  from a prompt \(x_0\), by sampling tokens one at a time from its distribution \(p(\cdot| x_{0:t-1}) \), 
Speculative decoding accelerates this process by introducing a smaller, faster draft model \(q(\cdot| x_{0:t-1})\) that proposes a block of \(N\) draft tokens \(\tilde{x}_{t:t+N-1}\).  For convenience, we write $p_t(\cdot) = p(\cdot | x_{0:t-1})$ and $q_t(\cdot) = q(\cdot | x_{0:t-1})$, when the prefix $x_{0:t-1}$ is clear from the context.
The target model then verifies in parallel by computing \(p(\tilde{x}_{t:t+N-1}| x_{0:t-1})\) and accepts \(\tilde{x}_t\) with probability
\begin{align*}
b_t(\tilde{x}_t) \;=\; \min\!\left\{1,\;{p_t(\tilde{x}_t)}/{q_t(\tilde{x}_t)}\right\}.
\end{align*}
If accepted, we set $x_t $ as $\tilde{x}_t$. Otherwise, the target samples a replacement token \(x_t\) from the fallback distribution 
\begin{equation*}
P_t(x)=\frac{[p_t(x) - q_t(x)]_+}{\sum_z [p_t(z) - q_t(z)]_+},
\end{equation*}
where $[x]_+ = \max\{0, x\}$.
Under standard conditions, the resulting stream of tokens is exactly distributed according to the target model~\citep{leviathan2023fast}.

Let $T_N$ denote the number of accepted tokens in a block of $N$ draft tokens.
Prior work~\citep{leviathan2023fast} shows that $\E[T_N] = ({1-\beta^{N+1}})/({1-\beta})$, where $\beta$ is the expected accepted probability and equals to \(1 - \TV(p,q) \). These guarantees, however, rely on the assumption that token verification is independent and identical (\textit{i.i.d.}), which have not considered the effect of the draft token positions in the draft sequence.

\textbf{Lossy speculative decoding and its Pareto front.}
A \textit{lossy} variant of speculative decoding replaces the standard rule with a relaxed accept rule \(b_t(\tilde{x}_t)\). Upon rejection, it samples from a fallback \(P_t(\cdot)\), inducing a relaxed token distribution \(\mathbb{P}(\cdot|{b_t},{P_t})\). The detailed procedure is given in Algorithm \ref{alg:lossy-sd}.
To increase acceptance relative to lossless speculative decoding, the accept rule should be more permissive than the standard rule.
The design objective of lossy speculative decoding is a multi‑objective problem that simultaneously minimizes (1)  the rejection probability \(\mathbb{P}_{b_t}(\mathrm{reject})\) and (2) the total variation bias \(\TV(\mathbb{P}(\cdot;{b_t},{P_t}),\, p)\). We formalize the trade‑off between acceptance and distribution bias via Pareto optimality. 

\begin{definition}[Pareto-front solutions for lossy speculative decoding]
We say that an accept--fallback pair $(b_t,P_t)$ \emph{dominates} $(b'_t,P'_t)$ if  $\mathbb{P}_{b_t}(\mathrm{reject}) < \mathbb{P}_{b'_t}(\mathrm{reject})$  and $\TV\!\left(\mathbb{P}(\cdot|{b_t},{P_t}),\, p\right) < \TV\left(\mathbb{P}(\cdot;{b'_t},{P'_t}),\, p\right).$    
A pair $(b_t,P_t)$ is \emph{Pareto-optimal} (non-dominated) if no other pair dominates it. 
The set of all Pareto-optimal pairs forms the \emph{Pareto front}.
\end{definition}

\begin{theorem}[\cite{yin2024theoretical}, Theorem 5; informal]\label{theorem:thm5}
An accept--fallback pair $(b_t,P_t)$ is Pareto-optimal if and only if it satisfies
\begin{equation}\label{eq:fornt}
\mathbb{P}_{b_t}(\mathrm{reject}) + \TV\!\left(\mathbb{P}(\cdot;{b_t},{P_t}),\, p\right) \;=\; \TV(q,\, p).
\end{equation}
\end{theorem}
Theorem~\ref{theorem:thm5} shows that reducing the rejection probability---and thus improving time efficiency---necessarily comes at the cost of increased distribution bias. Here, \emph{distribution bias} is the discrepancy between the token distribution induced by relaxed speculative decoding and the target-model distribution, measured (following prior work~\citep{yin2024theoretical}) by the total variation distance.

\section{METHODOLOGY} \label{sec:method}

\subsection{Static Ensemble}\label{sec:static-ensemble}

\textbf{Motivating example.} Figure~\ref{fig:acceptance_total_time} (a) studies the relation between inference time and acceptance rate. The results are obtained by generating 128 tokens on the CNNDM dataset, using three target/draft pairs from the Llama, Qwen, and Gemma families.
Across all pairs, we observe that wall-clock time decreases monotonically as acceptance rate increases. This highlights acceptance as the primary driver of speedup.

Motivated by this observation,  a straightforward idea to improve the acceptance rate is \emph{static ensemble}, which leverages a fixed ensemble of the target distribution $p_t(\cdot)$ and the draft distribution $q_t(\cdot)$ as the verification distribution. 
Given a constant weight $0 < w \leq 1$, define the ensemble distribution $\nu_t$ as:
\begin{equation} \label{eq:nu_t}
    \nu_t(\cdot) = w p_t(\cdot) + (1-w)q_t(\cdot).
\end{equation}
During verification, a draft token $\tilde{x}_t$ is accepted with probability 
$b_{\nu_t}(\tilde{x}_t)=\min\{1, \frac{\nu_t(\tilde{x}_t)}{q_t(\tilde{x}_t)}\}$.
If the token is rejected, we resample $x_t$ from the {fallback distribution}:
\[
P_{\nu_t}(x) = \frac{[\nu_t(x)-q_t(x)]_+}{\sum_{x'} [\nu_t(x')-q_t(x')]_+}.
\]
Following the analysis of~\cite{chen2023accelerating, leviathan2023fast}, we can show that this verification and resampling mechanism ensures that the resulting token distribution $x_t$ matches the intended ensemble distribution $\PP(x_t=x\mid x_{0:t-1})=\nu_t(x)$.

\begin{proposition}\label{prop:optimal}
The static ensemble achieves the Pareto front, i.e., $(\nu_t, P_{\nu_t})$ satisfies: 
\begin{equation*}
\PP_{\nu_t}(\mathrm{reject}) + \TV(\nu_t, p) = \TV(q, p).
\end{equation*}
Proof is deferred to Appendix~\ref{appex:proof_optimal}.
\end{proposition}

\begin{figure*}[!t]
\begin{mdframed}[linewidth=0.2pt, linecolor=gray,font=\rmfamily]
\textcolor{gray}{[Prompt] Each bird eats 12 beetles per day, ..., how many beetles are eaten each day?}

\textcolor{black}{Each jaguar eats 5 snakes per day, \textcolor{ForestGreen}{so} 6 jaguars will eat 6 ... per day. \textcolor{ForestGreen}{Each} bird eats 12 beetles day, and}

\phantom{Each jaguar eats 5 snakes per \,\,\quad}\textcolor{ForestGreen}{[and]} \textcolor{ForestGreen}{there ...}\textcolor{ForestGreen}{\checkmark} \quad\quad \quad\quad \quad\quad \qquad\quad \textcolor{ForestGreen}{[So]}\textcolor{ForestGreen}{, each bird...}\textcolor{ForestGreen}{\checkmark}

\textcolor{black}{\textcolor{ForestGreen}{there} are 90 birds, so in total they eat 12 * 90 = 1080 beetles per day.}

\textcolor{ForestGreen}{[90] birds eat 12 * 90 = 1080 beetles per day. \textcolor{ForestGreen}{\checkmark}}

\noindent\rule{\linewidth}{0.2pt}

\textcolor{gray}{[Prompt]\! Alexis is applying for a new job ... How much did Alexis pay for the shoes?}

\textcolor{black}{Alexis spent $\$30 + ... + \$18 = \textcolor{red}{\$143}$ on the items she has receipts for. \textcolor{red}{She has $\$16$} left from her budget, so} 

\textcolor{gray}{\phantom{Alexis spent $\$30 + ... + \$18 = $} \textcolor{red}{[\$153] on the ...}} \quad\quad \quad\quad \quad\qquad \quad\qquad  \textcolor{red}{she spent $\$49$. \textcolor{red}{{$\times$}}}

\textcolor{black}{she spent ... $\$184$ \textcolor{red}{in} total. Therefore, she spent $\$184 - \$143 =$ $\$41$ on the shoes.}

\textcolor{gray}{\phantom{she spent ... $\$184$} \textcolor{red}{[on] the shoes ...{$\times$}}}
\end{mdframed}
    \caption{Examples of speculative-decoding verification with accepted mismatches, illustrating not all draft--target mismatches are equally harmful.
    \textbf{Black} marks tokens both models agree on; \textcolor{ForestGreen}{\textbf{green}} marks accepted mismatches that still yield the correct answer; \textcolor{red}{\textbf{red}} marks accepted mismatches that lead to an incorrect answer.}
    \label{fig:example}
\end{figure*}

Together with Theorem \ref{theorem:thm5}, proposition~\ref{prop:optimal} demonstrates that the static ensemble exactly characterizes the Pareto-optimal tradeoff between acceptance probability and distributional bias. 
By varying the ensemble weight $w$, $(\nu_t, P_{\nu_t})$ traverses the entire Pareto front. 

We then quantify how the static ensemble increases the probability of accepting a single draft token.

\begin{lemma}\label{lem:linear}
Let draft and target models with distributions $q_t$ and $p_t$. For a constant $0 < w \leq 1$, define the accept rule $b_{\nu_t}(x) = \min\{1, \frac{\nu_t(x)}{q_t(x)}\}$. Then 
\begin{equation*}
\PP_{\nu_t}\!\big(\tilde{x}_t \mathrm{\; is\; accepted}\mid x_{0:t-1}\big)= 1 - w \, \TV(q,p).
\end{equation*}
The proof is deferred to Appendix~\ref{lem:expect-proof}.
\end{lemma}
Lemma~\ref{lem:linear} shows that static ensemble increases the acceptance probability from $1 - \TV(q,p)$ to $1 - w\,\TV(q,p)$. When $w=1$, it recovers the lossless speculative decoding; when $w<1$, it boosts acceptance.

\subsection{Dynamic Verification Relaxed Speculative Decoding}\label{sec:diversed}

\textbf{Motivating example.}
Figure~\ref{fig:example} illustrates speculative-decoding verification under accepted mismatches.
The example highlights that relaxing verification at different positions can have markedly different effects on generation quality: not all draft--target mismatches are equally harmful.
Tokens marked in green can be relaxed without degrading the final output, whereas relaxing the red tokens leads to an incorrect answer.
This observation suggests that verification relaxation should be context-dependent.
To increase acceptance while preserving correctness, the verifier should lean toward the target distribution for quality-critical tokens (e.g., tokens that can derail a mathematical solution), while allowing more aggressive relaxation---and thus higher draft acceptance---for less influential tokens.

We introduce a weighting scheme, where the weight $w_t$ is conditioned on the context $x_{0:t-1}$  at each time step $t$.
To represent the context, we build an ensemble head on top of draft and target models, leveraging both the draft model hidden representation $h_t^q$ and the target model hidden representation $h_t^p$, which together summarize $x_{0:t-1}$.
The ensemble weight is parameterized by a feed-forward neural network:
\begin{equation}\label{eq:fnn}
w_t = f_\theta(h_t^q, h_t^p).
\end{equation}
The resulting dynamic ensemble distribution becomes
\begin{equation}\label{eq:dynamic_ensemble}
\nu^{\theta}_t(x) = w_t p_t(x) + (1 - w_t) q_t(x).
\end{equation}
During verification, a draft token $\tilde{x}_t$ is accepted with this new probability  $b_{\nu_t}(x) =\min\{1, {\nu^{\theta}_t(x)}/{q_t(x)}\}$. Similar to  Section \ref{sec:static-ensemble}, the resulting token distribution  matches the intended ensemble distribution $\PP(x_t=x\mid x_{0:t-1})=\nu^\theta_t(x)$.

\textbf{Training.}
During training, the draft and target models process the same prompt, and the ensemble distribution $\nu^{\theta}_t$ is computed as in Equation~\eqref{eq:dynamic_ensemble}, from which the next token $x_t$ is sampled.
Token-level supervision is typically unavailable~\citep{bachmann2025judge}, so we train $\nu^{\theta}$ with reinforcement learning using a sequence-level reward $R(x_{1:T})$.
For example, in mathematical reasoning, $R=1$ if the final answer is correct and $R=0$ otherwise.
Let $D$ be the dataset of prompts $x_0$.
We optimize a regularized objective that balances task performance and acceptance:
\begin{equation}\label{eq:objective}
\small
\mathbb{E}_{x_0 \sim D,\, x_{1:T} \sim \nu^{\theta}(x_0)}
\big[
R(x_{1:T}) + \lambda{\sum_{t=0}^{T} \big(1 - \TV(q_t, \nu^{\theta}_t)\big)}
\big],
\end{equation}
where the regularizer $\sum_{t=0}^{T} \bigl(1 - \TV(q_t, \nu^{\theta}_t)\bigr)$ promotes higher acceptance and $\lambda>0$ controls its strength. This regularization also discourages degenerate solutions that collapse to using only the target model ($w=0$) or only the draft model ($w=1$).
We optimize objective~\eqref{eq:objective} using the recent REINFORCE++ algorithm~\citep{hu2025reinforce++}, which stabilizes training through clipping and global batch averaging.

\begin{algorithm*}[t]
\caption{General Framework for Lossy Speculative Decoding.}
\label{alg:lossy-sd}
\begin{algorithmic}[1]
\REQUIRE Maximum generation length $T$; draft length $N$; prompt $x_0$; draft model distributions $\{q_t\}$; target model distributions $\{p_t\}$;
acceptance rule $b_t$; fallback distribution $\{P_t\}$.
\STATE Initialize $n \gets 1$;
\WHILE{$n \le T$}
\FOR{$t=n,\ldots,n+N$}
  \STATE  sample
         $\tilde x_t \sim q_t(\cdot \mid x_{0:n-1},\,\tilde x_{n:t-1})$ 
         and save $q_t$ into cache.   \COMMENT{draft rollout}
\ENDFOR
  \FOR{$t=n,\ldots,T$}
  \STATE  compute  $p_t(\cdot \mid x_{0:n-1},\,\tilde x_{n:t-1})$. \COMMENT{target parallel verification (Section~\ref{sec:method})}
  \ENDFOR
  \FOR{$t=n\ldots n+N$}
    \STATE draw $u \sim \mathrm{Uniform}(0,1)$.
    \IF{$u \le b_t(\tilde x_t \mid x_{1:n-1},\,\tilde x_{n:t-1})$}
      \STATE \textbf{accept} $\tilde x_t$; \quad set  $x_n \gets \tilde x_t$; \quad update  $n \gets n+1$.
    \ELSE
      \STATE \textbf{reject}; \, sample $x_n \sim P_t$;\quad  update $n \gets n+1$; \quad \textbf{break}. \COMMENT{discard $\tilde x_{t+1:n+N-1}$ and resume}
    \ENDIF
  \ENDFOR
\ENDWHILE
\RETURN sampled sequence $x_{1:T}$.
\end{algorithmic}
\end{algorithm*}

\textbf{Inference.} At inference, the draft model autoregressively propose $N$ draft tokens $\tilde{x}_{t:t+N-1}$ along with hidden states $h_{t:t+N-1}^q$ and distributions $q_{t:t+N-1}$. 
The target model evaluates these tokens in parallel to produce $h_{t+1:t+N}^p$ and $p_{t:t+N-1}$. The ensemble head then computes weights  $w_{t:t+N-1}$ via Equation~\eqref{eq:fnn} and the corresponding dynamic ensemble distributions $ \nu_{t:t+N-1}^{\theta}(\cdot)$ via Equation~\eqref{eq:dynamic_ensemble} in parallel.
Thus, the relaxed accept rule and fallback distribution are computed as follows:
\begin{equation}\label{eq:diversed_b_p}
\begin{aligned}
 b_{\nu^{\theta}_t}(\tilde{x}_t)& = \min\left\{1, {\nu^{\theta}_t(\tilde{x}_t)}/{q_t(\tilde{x}_t)} \right\},\\
P_{\nu^{\theta}_t}(x) &= \frac{[\nu^{\theta}_t(x) - q_t(x)]_+}{\sum_z [\nu^{\theta}_t(z) - q_t(z)]_+}.
\end{aligned}
\end{equation}
The inference pipeline is in Appendix Algorithm~\ref{alg:inference_diversed}.

\section{THEORETICAL ANALYSIS} \label{sec:theory}

This section derives the exact expected acceptance rate given a draft length $N$. As each accepted token avoids an expensive call to the target model, the key factor driving speedup in speculative decoding is the number of accepted draft tokens~\citep{leviathan2023fast}.
However, this was only approximated under an \textit{i.i.d.} assumption on acceptance probabilities. In practice, this assumption is violated: the acceptance probability of a draft token is highly related to the position of the token in the draft sequence~\citep{huang2024specdec++, Mamou2024DISCO}. 
We therefore provide a precise characterization of the expected acceptance length.

Assume we are in the first round, where the draft model needs to draw a sequence $\hat{x}_{1: N}$ from its probability distribution $q$ given the prompt $x_0$. Then the verifier evaluates whether to accept or reject the $t$-step token.
Let $A_t$ be the acceptance indicator, where $A_t = 1$ if the draft token $\tilde{x}_t$ is accepted given the history $x_{0:t-1}$; otherwise $A_t = 0$. The acceptance length $T_N$ is the number of consecutive draft tokens accepted before the first rejection. When all $N$ are accepted, $T_N=N$.
 Next, we characterize the distribution of the accepted tokens, which will be used for computing the expected accepted length.

\begin{lemma}[Distribution of Accepted Tokens]\label{lem:accept_drift}
    Consider a draft and  target model  with corresponding conditional distributions $q(\cdot |  x_{0:t-1})$ and $p(\cdot |  x_{0:t-1})$ with prefix $x_{0:t-1}$. Given that a token is accepted by speculative decoding, its distribution follows the min-kernel: 
    \begin{align*}
    \PP&(x_t = x |  x_t  \textrm{\; is accepted},x_{0:t-1}) =\\
    &\quad \frac{\min\{q(x |  x_{0:t-1}), \, p(x |  x_{0:t-1})\}}
        {1 - \TV\big(q(\cdot |  x_{0:t-1}), \, p(\cdot |  x_{0:t-1})\big)} 
        \eqqcolon  \tilde{m}(x |  x_{0:t-1})
    \end{align*}
        Proof is deferred to Appendix~\ref{appex:proof_bayes}.
\end{lemma}

Lemma~\ref{lem:accept_drift} shows that accepted tokens are not distributed according to the target distribution, but instead follow the min-kernel $\tilde{m}_t(x)$. For example, if the draft generates $\tilde{x}_{1:N}$ and $x_{1:t}=\tilde{x}_{1:t}$ are all accepted, then $x_{1:t}$ comes from $\tilde{m}$. This is important when characterizing the acceptance rate, as the precondition of draft token $\tilde{x}_t$ being accepted is only when all previous drafts $\tilde{x}_{1:t-1}$ are accepted. We conduct more experiments to illustrate Lemma \ref{lem:accept_drift}; the details can be found in the Appendix \ref{appx:more-exp}.

\nj{\citet{yin2024theoretical} quantified that the number of rejections is $\sum_{t=1}^N\E_{x_{1:t-1}\sim q}[\TV (q(\cdot|x_{0:t-1}),p(\cdot|x_{0:t-1}))]$. %
However, this view ignores the position of rejection. 
In practice, a rejection at the first draft token and a rejection at the \((N-1)\)-th token have very different costs: in the former case, generating a block of \(N\) proposals requires at least \(2N-1\) draft calls plus one target call, whereas in the latter it needs only \(N\) draft calls and one target call. 
\cite{yin2024theoretical} treats both cases as equal cost because it counts only target invocations, yet the acceptance lengths differ markedly (\(T_N=1\) versus \(T_N=N-1\)), leading to different wall-clock times. Motivated by this, we derive a more precise expression for the expected acceptance length.}

\begin{theorem}[Expected Acceptance Length]\label{theom:expected_accept}
    Consider a draft  and target model with corresponding conditional distributions $q(\cdot |  x_{0})$, $p(\cdot |  x_{0})$. Given draft length $N$, 
    the expected acceptance length $\E[T_N]$ is
    \begin{align*}
    \sum_{t=1}^N \E_{x_{1:t-1}\sim \tilde{m}} \left[\prod_{i=0}^{t}\left(1 - \TV\left(q(\cdot | x_{0:i}),  p(\cdot | x_{0:i})\right)\right)\right].
    \end{align*}
where the expectation is taken w.r.t. the min-kernel $\tilde m$ in Lemma~\ref{lem:accept_drift}.
\begin{proof}[Proof Sketch]
We directly characterize $\E[T_N |  \tilde{x}_{1:N}]$ by tower law. As the fact that $T_N$ is non-negative, we know that $\E[T_N | \tilde{x}_{1:N}] = \sum_{t=1}^N \PP(T_N > t |  \tilde{x}_{1:t-1})$, 
and $\PP(T_N > t |  \tilde{x}_{1:t-1})$ is the probability that the first rejection happens after the step $t$ given a draft sequence $\tilde{x}_{1:t-1}$. We then apply the induction rule to compute $\PP(T_N > t |  \tilde{x}_{1:t-1})$, as it can be written as $P(\text{Accept}\,\tilde{x}_t |  \tilde{x}_{1:t-1} \,\text{is accepted})\times P(\text{Accept} \, \tilde{x}_{1:t-1} |  \tilde{x}_{1:t-1})$, where the min-kernel distribution $\tilde{m}$  characterizes the accepted tokens. The full proof is presented in Appendix~\ref{appex:proof_expected_accept}.
\end{proof}
\end{theorem}

Theorem \ref{theom:expected_accept} characterizes the expected accepted draft length using sequential dependent draft and distribution $\{q_t\}$ and $\{p_t\}$ for speculative decoding on sequence length $N$. Moreover, as $1-\TV(q_t, p_t) \le 1$, the probability of accepting the draft token $\tilde{x}_t$ exponentially decrease with the time step $t$.

\section{RELATED WORKS}\label{sec:relate}

\textbf{Speculative decoding (SD)} accelerates autoregressive generation by having a small \emph{draft} model propose tokens that a larger \emph{target} model verifies in parallel.
\citet{leviathan2023fast} introduced lossless SD with exact target-distribution matching.
These methods achieve about \(2\times\) to \(3\times\) wall-clock gains but expose a verification bottleneck.
EAGLE~\citep{li2024eagle}, Medusa~\citep{Cai2024Medusa}, and LayerSkip~\citep{Elhoushi2024LayerSkip} restructure the target model to reuse its computation for drafting, reducing the overhead of a separate draft model.
DISCO~\citep{Mamou2024DISCO} and SpecDec++~\citep{huang2024specdec++} adapt the draft length dynamically, avoiding low acceptance when sequences are too long.
Speculative Cascades~\citep{narasimhan2024faster} unify cascaded inference with SD to trade off quality and cost.
Analyses include a broad empirical study of SD bottlenecks~\citep{yan2024decoding} and a theoretical treatment that formalizes the efficiency–quality trade-off~\citep{yin2024theoretical}.

\begin{figure*}[!t]
	\centering
    \includegraphics[width=\textwidth]{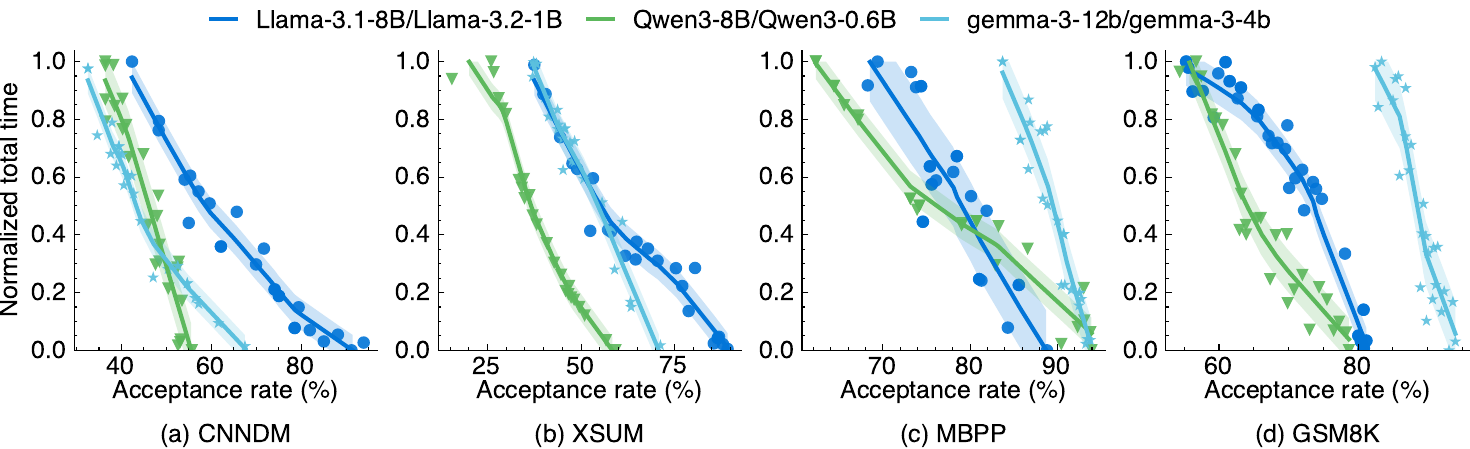}
    \caption{Normalized total time (lower is better) versus acceptance rate (higher is better) under speculative decoding with draft length $N{=}5$. Model pairs with higher acceptance achieve lower decoding time. Results on the CNNDM, XSUM, MBPP, and GSM8K datasets for three target/draft pairs. Additional results for draft lengths $N{=}3$ and $N{=}7$ are provided in Appendix Figures~\ref{fig:acceptance_total_time_len3} and~\ref{fig:acceptance_total_time_len7}.}
\label{fig:acceptance_total_time}
\end{figure*}

\begin{table*}[!ht]
\centering
\caption{Benchmark acceptance rate and generation quality. Columns group results by task while panels (a)--(c) correspond to different \textit{Target}/\textit{Draft} model pairs. All experiments use temperature $1$ and draft length $N=5$. 
Our \method achieves a higher acceptance rate and also maintains comparable quality.
Additional results for temperatures $0$ and $0.1$ are reported in Appendix Tables~\ref{tab:main_acceptancerate_false} and~\ref{tab:main_acceptancerate_small}.
}
\label{tab:main_acceptancerate}
\resizebox{\linewidth}{!}{%
\begin{tabular}{l|cc|cc|cc|cc}
\toprule

&\multicolumn{2}{c}{\textbf{GSM8K}} &
\multicolumn{2}{|c}{\textbf{CNNDM}} &
\multicolumn{2}{|c}{\textbf{XSum}} &
\multicolumn{2}{|c}{\textbf{MBPP}} \\
\textbf{Method} &
\makecell{Accept \\Rate ($\uparrow$)} & \makecell{Quality\\Accuracy\,($\uparrow$)} &
\makecell{Accept \\Rate($\uparrow$)} & \makecell{Quality\\ROUGE-2\,($\uparrow$)} &
\makecell{Accept \\Rate($\uparrow$)} & \makecell{Quality\\ROUGE-2\,($\uparrow$)} &
\makecell{Accept \\Rate($\uparrow$)} & \makecell{Quality\\pass@1\,($\uparrow$)} \\
\midrule
Autoregressive & NA & 67$\%$ & NA & 9.86 & NA & 7.03 & NA & 53$\%$\\
\midrule
SD & 44.60$\%$ & 67$\%$ & 21.60$\%$ & 9.46 & 20.44$\%$ & 7.09 & 26.30$\%$ & 53$\%$ \\
SD (Lossy) & 59.81$\%$ & 66$\%$ & 38.86$\%$ & 10.51 & 40.91$\%$ & 7.96 & 66.75$\%$ & 49$\%$ \\
SpecCascade & 61.53$\%$ & 67$\%$ & 47.29$\%$ & 11.74 & 43.95$\%$ & 7.61 & 73.92$\%$ & 52$\%$ \\
Static Ensemble & 69.49$\%$ & 66$\%$ & 61.06$\%$ & 11.46 & 51.58$\%$ & 7.22 & 68.70$\%$ & 52$\%$ \\
\method\,(ours)  & \textbf{72.61$\%$} & 67$\%$ & \textbf{69.96$\%$} & {12.11} & \textbf{70.53$\%$} & {7.23} & \textbf{85.03$\%$} & {53$\%$} \\
\bottomrule
\multicolumn{9}{c}{ \textbf{(a)} Target/Draft model pair is {Llama-3.1-8B}/{Llama-3.2-1B}.}\\
\toprule

Autoregressive & NA & 90$\%$ & NA & 9.97 & NA & 4.90 & NA & 55$\%$ \\
\midrule
SD & 59.58$\%$ & 91$\%$ & 35.26$\%$ & 9.85 & 15.60$\%$ & 4.95 & 58.72$\%$ & 55$\%$ \\
SD (lossy) & 60.10$\%$ & 86$\%$ & 36.40$\%$ & 10.53 & 19.56$\%$ & 4.90 & 60.85$\%$ & 53$\%$ \\

 SpecCascade & 58.65$\%$ & 86$\%$ & 38.13$\%$ & 10.85 & 19.45$\%$ & 4.97 & 71.27$\%$ & 56$\%$ \\
Static Ensemble & 67.58$\%$ & 86$\%$ & 41.07$\%$ & 10.88 & 21.24$\%$ & 4.98 & 76.11$\%$ & 55$\%$ \\
\method\,(ours)  & \textbf{76.48$\%$} & 88$\%$ & \textbf{46.59$\%$} & 10.97 & \textbf{47.01$\%$} & 5.01 & \textbf{81.46$\%$} & 59$\%$ \\
\bottomrule
  \multicolumn{9}{c}{ \textbf{(b)} Target/Draft model pair is {Qwen3-8B}/{Qwen3-0.6B}.}  \\
\toprule
Autoregressive & NA & 93$\%$ & NA & 9.01 & NA & 8.31 & NA & 68$\%$ \\
\midrule
SD & 84.15$\%$ & 92$\%$ & 40.39$\%$ & 9.06 & 35.76$\%$ & 8.27 & 83.25$\%$ & 67$\%$ \\
SD (Lossy) & 85.02$\%$ & 90$\%$ & 45.43$\%$ & 10.62 & 39.48$\%$ & 8.38 & 86.69$\%$ & 65$\%$ \\
SpecCascade & 84.43$\%$ & 92$\%$ & 51.44$\%$ & 10.42 & 39.42$\%$ & 8.27 & 83.76$\%$ & 67$\%$ \\
Static Ensemble & 87.62$\%$ & 91$\%$ & 54.48$\%$ & 10.84 & 61.42$\%$ & 7.90 & 86.46$\%$ & 66$\%$ \\
\method\,(ours) & \textbf{90.70$\%$} & {92$\%$} & \textbf{66.90$\%$} & 10.86 & \textbf{63.38$\%$} & {7.22} & \textbf{90.23$\%$} & {67$\%$} \\
\bottomrule
\multicolumn{9}{c}{ \textbf{(c)} Target/Draft model pair is {Gemma-3-12B}/{Gemma-3-4B}.}
\end{tabular}%
}
\end{table*}

\textbf{Improving acceptance and adaptive drafting.}
EAGLE~\citep{li2024eagle} rethinks feature uncertainty for speculative sampling, while Medusa~\citep{Cai2024Medusa} equips LLMs with multiple decoding heads to explore continuation trees jointly.
Dynamic lookahead methods such as DISCO~\citep{Mamou2024DISCO} and SpecDec++~\citep{huang2024specdec++} adaptively choose draft lengths, proving threshold-type optimal policies and yielding $7$--$11\%$ further speedups.
~\cite{yan2024decoding} provides a systematic empirical study of bottlenecks across 350 experiments.
SpecInfer~\citep{Miao2023SpecInfer} introduces token-tree verification for parallel throughput.
SSSD and EMS-SD~\citep{Ni2025EMSSD} extend SD to continuous batching and multi-sample inference, respectively, avoiding padding overhead.
Speculative cascades~\citep{narasimhan2024faster} unify cascaded inference with SD to trade off quality and cost.
Draft \& Verify~\citep{Zhang2024DraftVerify} eliminates the separate draft model by skipping intermediate layers.
LayerSkip~\citep{Elhoushi2024LayerSkip} similarly reuses early layers for drafting and later layers for verification, reducing memory while maintaining fidelity.
NEST~\citep{Li2024NEST} augments SD with nearest-neighbor retrieval to improve factuality and attribution, while retaining acceleration benefits.
~\cite{yin2024theoretical} formalize SD under a Markov abstraction, deriving a linear Pareto frontier between efficiency and fidelity.
This theory motivates the exploration of relaxed verification.
Recent work questions the need for exact alignment.
Judge Decoding~\citep{bachmann2025judge} employs an LLM-as-judge to accept semantically valid but non-identical tokens.

Our \method advances this line by proposing a \emph{dynamic ensemble verifier} that adaptively blends target and draft distributions, going beyond the Pareto front obtained by a static ensemble.

\section{EXPERIMENTS} \label{sec:exp}
In experiments, we show \method delivers substantially higher acceptance rates, while preserving task quality, across all model pairs and benchmarks.
We further show that wall-clock speedup closely correlated to {acceptance rate}: increases in acceptance rate reliably translate into end-to-end latency reductions. Finally, we observe a clear efficiency–quality Pareto front, and \method consistently attains a strong balance between the two.

\subsection{Experimental Settings}

For the choice of  LLM models, we consider three target~/~draft model pairs:
(1) Llama-3.1-8B-Instruct~/~Llama-3.2-1B-Instruct~\citep{dubey2024llama},
(2) Qwen3-8B~/~Qwen3-0.6B~\citep{yang2025qwen3}, (3) Gemma-3-12B-It/Gemma-3-4B-It~\citep{team2025gemma}.

In terms of  benchmark datasets, we consider:
(1) GSM8K for mathematical reasoning~\citep{cobbe2021gsm8k},
(2) CNNDM for news summarization~\citep{see-etal-2017-get},
(3) XSum for extreme summarization~\citep{Narayan2018DontGM},
and (4) MBPP for Python programming~\citep{austin2021program}.

For the choice of baselines, we include (1) lossless speculative decoding methods, including standard speculative decoding~\citep{leviathan2023fast, chen2023accelerating} and static ensemble, and (2) representative lossy speculative decoding methods, including SD (Lossy)~\citep{leviathan2023fast} and SpecCascade~\citep{narasimhan2024faster}.
Additional experimental details are provided in Appendix~\ref{appx:exp-set}.

\subsection{Experimental Result and Analysis}
\textbf{Benchmark inference time and acceptance rate.} Figure \ref{fig:acceptance_total_time} reports the normalized wall-clock time needed for different model pairs to complete generation vs the acceptance rate in speculative decoding. 
We measure wall-clock time per model pair as the time required to generate a fixed number of tokens: 384 tokens for GSM8K, 128 for CNNDM and XSum, and 512 for MBPP.
We find that the wall-clock time decreases monotonically as the acceptance rate increases.

\textbf{Benchmark acceptance rate and generation quality.} In Table \ref{tab:main_acceptancerate}, we report the acceptance rate along with the generation quality across different model pairs and various tasks. For the baselines, lossy speculative decoding, SpecCascade, and static ensemble are reported such that the generation quality matches the target model quality.  

Overall, \method consistently achieves a higher acceptance rate than lossy speculative, SpecCascade, and static ensemble.
On the {Llama-3.1-8B}/{Llama-3.2-1B} pair, \method consistently improves the acceptance rate by at least $28\%$ across all tasks while matching the target model generation quality.
This is because the draft model also generates high-quality completions for the summarization, and \method could effectively accept those high-quality draft tokens.

\begin{figure}
    \centering
    \includegraphics[width=\linewidth]{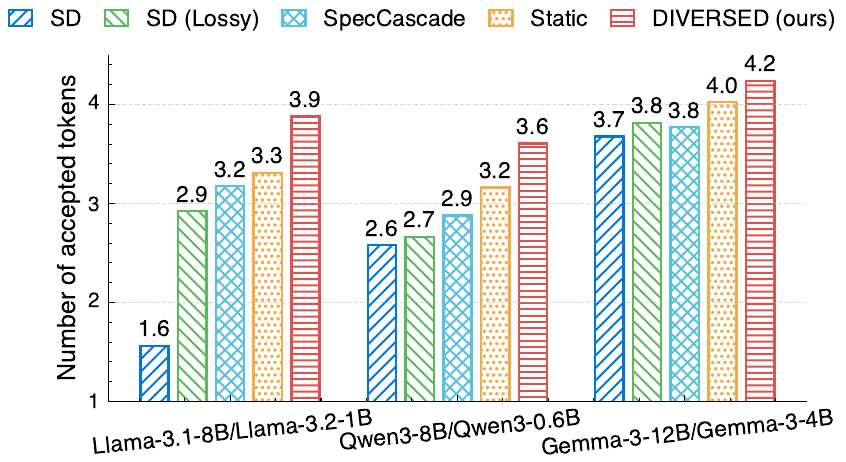}
    \caption{\method consistently accepts a larger average number of draft tokens per verification round than all baselines (draft length $N{=}5$, temperature $1$).
    Appendix Figure~\ref{fig:token_speed_up_temp_small} shows result for  temperature $0.1$.}
    \label{fig:token_speed_up}
\end{figure}

\textbf{Benchmark speedup.} 
Figure~\ref{fig:token_speed_up} summarizes speedup by reporting the average number of draft tokens accepted per verification round for each \textit{Target}/\textit{Draft} pair, averaged over all datasets and generations.

\method achieves the highest accepted-token count across all pairs, reflecting its adaptive, token- and context-dependent blending of the draft and target distributions: it expands the acceptance region for high-confidence draft tokens while controlling distributional deviation from the target.

\textbf{Benchmark inference time and generation quality.}
Figure~\ref{fig:llama_cnndm_comp} plots summarization quality (ROUGE-2) against end-to-end inference time. Methods are preferred when they achieve \emph{higher} ROUGE-2 at \emph{lower} latency.
Compared with autoregressive decoding and standard speculative decoding, the static ensemble traces a Pareto front that provides consistent improved time--quality tradeoffs.
Notably, \method improves upon this frontier, achieving higher quality at comparable latency and/or lower latency at comparable quality, thereby outperforming the Pareto-optimal set induced by the static ensemble.

\begin{table}[!t]
\centering
\caption{Comparison of finetuned draft vs.\ \method at temperature $1$. Fine-tuned draft model yields higher acceptance rates but lower task performance.}
\label{tab:zoro_T1_cnndm}
\resizebox{\linewidth}{!}{%
\begin{tabular}{l|l|cc}
\hline
\textbf{Train/Test} &
\multirow{3}*{\textbf{Method}} &
\multicolumn{2}{c}{\textbf{CNNDM}} \\
\textbf{Dataset}&&
Accept  & ROUGE-2 \\
&& Rate ($\uparrow$) & ($\uparrow$) \\
\hline
 CNNDM & SD & 21.06$\%$ & 9.46 \\
 {\small CNNDM/CNNDM} & \method  & \textbf{69.96$\%$} & \textbf{12.11} \\
 {\small GSM8K/CNNDM} & \method  & 51.07$\%$ & 8.63 \\
\hline
\multicolumn{4}{c}{\textbf{(a)} Draft/Target: Llama-3.2-1B / Llama-3.1-8B} \\
\hline
 CNNDM & SD & 35.26$\%$ & 9.85 \\
 {\small CNNDM/CNNDM} & \method  & \textbf{46.59$\%$} & \textbf{10.97} \\
 {\small  GSM8K/CNNDM }& \method & \textbf{62.55$\%$} & 8.13 \\
\hline
\multicolumn{4}{c}{\textbf{(b)} Draft/Target: Qwen3-0.6B / Qwen3-8B} \\
\hline
 &
 &
\multicolumn{2}{c}{\textbf{GSM8K}} \\
&&
Accept  & Accuracy \\
&& Rate ($\uparrow$) & ($\uparrow$) \\
\hline
 GSM8K & SD & 44.60$\%$ & 67$\%$  \\
 GSM8K/GSM8K & \method & \textbf{72.61$\%$} & \textbf{67$\%$}  \\
 CNNDM/GSM8K & \method & 72.57$\%$ & 62$\%$  \\
\hline
\multicolumn{4}{c}{\textbf{(c)} Draft/Target: Llama-3.2-1B / Llama-3.1-8B} \\
\hline
 GSM8K & SD & 59.58$\%$ & \textbf{91$\%$}  \\
{\small GSM8K/GSM8K} & \method & \textbf{76.48$\%$} & 88$\%$  \\
  {\small CNNDM/GSM8K} & \method & 74.30$\%$ & 78$\%$  \\
\hline
\multicolumn{4}{c}{\textbf{(d)} Draft/Target: Qwen3-0.6B / Qwen3-8B} \\
\end{tabular}%
}
\end{table}

\subsection{Case Studies}

\textbf{Benchmark Zero-shot Performance.}  We train \method on GSM8K and test the trained model on CNNDM, and vice versa (i.e., \method trained on CNNDM and tested on GSM8K).  %

\begin{table*}[!ht]
\centering
\caption{Comparison of finetuned draft vs. \method at two temperatures $T=1.0, 0.1$. Columns are grouped by task. Each block compares SD, SD (Finetuned draft), and DIVERSED for a given \textit{Draft/Target} pair. It shows fine-tuning does not reliably improve acceptance.}
\label{tab:finetune_vs_diversed_T1}
\begin{tabular}{l|c|lc|lc|lc}
\toprule
  &\multirow{3}*{\textbf{Temperature}} &
\multicolumn{2}{c|}{\textbf{GSM8K}} &
\multicolumn{2}{c|}{\textbf{XSum}} &
\multicolumn{2}{c}{\textbf{MBPP}} \\

 & 
& \makecell{Accept \\Rate ($\uparrow$)} & \makecell{Accuracy\\($\uparrow$)} &
 \makecell{Accept \\Rate ($\uparrow$)} & \makecell{ROUGE-2\\($\uparrow$)} &
\makecell{Accept \\Rate ($\uparrow$)} & \makecell{Pass@1\\($\uparrow$)} \\
\hline

SD&  \multirow{3}*{$1.0$} & 44.60$\%$ & 67$\%$ & 20.44$\%$ & 7.09 & 26.30$\%$ & 53$\%$ \\
SD (Finetuned draft) & & 66.44$\%$ & 67$\%$ & 35.41$\%$ & 6.95 & 47.07$\%$ & 53$\%$ \\
\method (Ours) & & \textbf{72.61$\%$} & 67$\%$ & \textbf{70.53$\%$} & \textbf{7.23} & \textbf{85.03$\%$} &  53$\%$ \\
\hline
\multicolumn{8}{c}{\textbf{(a)} Draft/Target model pair is Llama-3.2-1B / Llama-3.1-8B.} \\
\hline

 SD & \multirow{3}*{$1.0$}& 59.58$\%$ & 91$\%$ & 15.60$\%$ & 4.95 & 58.72$\%$ & 55$\%$ \\
 SD (Finetuned draft) & & 57.54$\%$ & 90$\%$ & 26.00$\%$ & 4.66 & 11.70$\%$ & 55$\%$ \\
 \method (Ours) & & \textbf{76.48$\%$} & 88$\%$ & \textbf{47.01$\%$} & \textbf{5.01} & \textbf{81.46$\%$} & \textbf{59$\%$} \\
\hline
\multicolumn{8}{c}{\textbf{(b)} Draft/Target  model pair is Qwen3-0.6B / Qwen3-8B.} \\

 \hline
SD &\multirow{3}*{$0.1$}&  66.40$\%$ & 76$\%$ & 17.69$\%$ & 7.82 & 72.94$\%$ & 61$\%$ \\
SD (Finetuned draft) & & 52.81$\%$ & 76$\%$ & 18.81$\%$ & \textbf{7.97} & 74.32$\%$ & 61$\%$ \\
 \method (Ours)  & & \textbf{89.14$\%$} & 76$\%$ & \textbf{64.75$\%$} & 7.52 & \textbf{89.78$\%$} & \textbf{62$\%$} \\
 \hline
\multicolumn{8}{c}{\textbf{(c)} Draft/Target model pair is Llama-3.2-1B / Llama-3.1-8B.} \\
\hline
SD &\multirow{3}*{$0.1$}&  58.83$\%$ & 92$\%$ & 16.45$\%$ & 5.26 & 83.85$\%$ & 67$\%$ \\
 SD (Finetuned draft) &&  60.45$\%$ & 92$\%$ & 18.16$\%$ & 5.62 & 48.27$\%$ & 67$\%$ \\
  \method (Ours)  &&  \textbf{90.26$\%$} & 92$\%$ & \textbf{71.22$\%$} & \textbf{7.84} & \textbf{91.23$\%$} & 67$\%$ \\
  \hline
\multicolumn{8}{c}{\textbf{(d)} Draft/Target model pair is Qwen3-0.6B / Qwen3-8B.} 
\end{tabular}%
\end{table*}

Compared with standard speculative decoding, \method can be trained on one dataset and evaluated on another (e.g., trained on GSM8K and tested on CNNDM, or vice versa). In these cross-dataset settings, we observe higher acceptance rates but degraded task performance. This is expected: when training and testing on the same dataset, \method can learn task-specific relaxation that better preserves solution quality. These results further suggest that, to maintain strong performance, relaxation of the rejection rule should be task-dependent rather than universal.

\textbf{Fine-tuned draft model \textit{vs.} dynamic ensemble.}
We fine-tune the draft model separately on each dataset and then use the fine-tuned draft within standard speculative decoding. Table~\ref{tab:finetune_vs_diversed_T1} reports results on GSM8K, XSum, and MBPP. Overall, fine-tuning yields accuracy/pass@1 comparable to \method, but \method consistently achieves the highest acceptance rate. Importantly, fine-tuning does not reliably improve acceptance: for the Llama pair, acceptance typically increases after fine-tuning, whereas for Qwen3 it decreases in most settings. These findings suggest that acceptance in speculative decoding is governed primarily by \emph{distributional alignment} between the draft and target conditionals, rather than by the draft model’s standalone task performance. %

\begin{figure}[!t]
\centering
    \includegraphics[width=.97\linewidth]{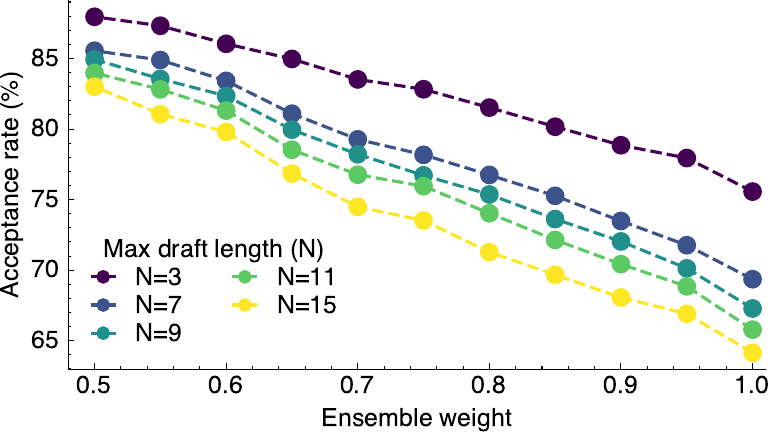}
    \caption{Acceptance rate decreases approximately linearly with the ensemble weight for static ensemble on CNNDM for the \textit{Target}/\textit{Draft} pair {Llama-3.1-8B-Instruct}/{Llama-3.2-1B-Instruct} at temperature $0$.}
    \label{fig:llama_cnndm}
\end{figure}

\textbf{Impact of the static ensemble weight.}
Figure~\ref{fig:llama_cnndm} shows that the acceptance rate decreases as the ensemble weight $w$ increases. In the limit $w \to 1$, static-ensemble verification reduces to standard (lossless) speculative decoding. We report results on CNNDM using the Target/Draft pair Llama-3.1-8B-Instruct / Llama-3.2-1B-Instruct at temperature $0$. This trend is consistent across draft lengths $N$ and aligns with Lemma~\ref{lem:linear}, which predicts that acceptance decreases monotonically with $w$.

\section{CONCLUSION}
We studied lossy speculative decoding as a practical approach to accelerate large language model inference by increasing draft-token acceptance while preserving generation quality.
Our empirical results show that optimal acceptance behavior is inherently context- and task-dependent, limiting the effectiveness of static verification rules.

To address this, we propose \method, a learned \emph{dynamic} ensemble verifier that adapts to the local context at each decoding step and is trained in a task-specific manner.
On the theory side, we derive an exact expression for the expected acceptance length without i.i.d.\ assumptions, yielding a more precise characterization than prior analyses.
Empirically, we find that higher acceptance reliably translates into lower wall-clock latency, and \method consistently achieves higher acceptance at comparable or improved task quality.

Future work includes extending relaxed verification from token-level to block-level decisions and exploring cross-task transfer for the learned dynamic verifier.

\paragraph{Acknowledgments} Ziyi Wang conducted this work during an Amazon internship. Nan Jiang is supported by TACC CCR25054.

\newpage
\bibliography{reference}

\section*{Checklist}

\begin{enumerate}

  \item For all models and algorithms presented, check if you include:
  \begin{enumerate}
    \item A clear description of the mathematical setting, assumptions, algorithm, and/or model. [Yes]
    \item An analysis of the properties and complexity (time, space, sample size) of any algorithm. [Yes]
    \item (Optional) Anonymized source code, with specification of all dependencies, including external libraries. [Yes]
  \end{enumerate}

  \item For any theoretical claim, check if you include:
  \begin{enumerate}
    \item Statements of the full set of assumptions of all theoretical results. [Yes]
    \item Complete proofs of all theoretical results. [Yes]
    \item Clear explanations of any assumptions. [Yes]     
  \end{enumerate}

  \item For all figures and tables that present empirical results, check if you include:
  \begin{enumerate}
    \item The code, data, and instructions needed to reproduce the main experimental results (either in the supplemental material or as a URL). [Yes]
    \item All the training details (e.g., data splits, hyperparameters, how they were chosen). [Yes]
    \item A clear definition of the specific measure or statistics and error bars (e.g., with respect to the random seed after running experiments multiple times). [Yes]
    \item A description of the computing infrastructure used. (e.g., type of GPUs, internal cluster, or cloud provider). [Yes]
  \end{enumerate}

  \item If you are using existing assets (e.g., code, data, models) or curating/releasing new assets, check if you include:
  \begin{enumerate}
    \item Citations of the creator If your work uses existing assets. [Yes]
    \item The license information of the assets, if applicable. [Yes]
    \item New assets either in the supplemental material or as a URL, if applicable. [No]
    \item Information about consent from data providers/curators. [Yes]
    \item Discussion of sensible content if applicable, e.g., personally identifiable information or offensive content. [Not Applicable]
  \end{enumerate}

  \item If you used crowdsourcing or conducted research with human subjects, check if you include:
  \begin{enumerate}
    \item The full text of instructions given to participants and screenshots. [Not Applicable]
    \item Descriptions of potential participant risks, with links to Institutional Review Board (IRB) approvals if applicable. [Not Applicable]
    \item The estimated hourly wage paid to participants and the total amount spent on participant compensation. [Not Applicable]
  \end{enumerate}

\end{enumerate}

\clearpage
\appendix
\onecolumn
\aistatstitle{Supplementary Materials: DIVERSED: Relaxed Speculative Decoding
via Dynamic Ensemble Verification}

\section{Training and Inference Pipeline of DIVERSED} \label{app:implementation_diversed_inference}

\begin{algorithm}[!htbp]
\caption{Inference pipeline of \method.}
\label{alg:inference_diversed}
\begin{algorithmic}[1]
\REQUIRE Horizon $T$; draft length $N$; prompt $x_0$; draft model $\{q_t\}$; target model $\{p_t\}$.
\ENSURE Sequence $x_{1:T}$.
\STATE $n \gets 1$;
\WHILE{$n \le T$}
   \FOR{$t=n$ \TO $n+N$  }
   \STATE sample $\tilde x_t \sim q_t(\cdot | x_{0:n-1},\,\tilde x_{n:t-1})$
          \COMMENT{draft rollout} 
          \STATE cache hidden states $h_{n:n+N}^q$ and distribution $q_{n:n+N}$.
    \ENDFOR
  \FOR{$t=n$ \TO $n+N$}
  \STATE  compute
         $p_t(\cdot | x_{0:n-1},\,\tilde x_{n:t-1})$, and cache hidden states $h_{n:n+N}^p$. \COMMENT{target verification} 
    \ENDFOR
  \STATE
  compute the relaxed accept rule $(b_{\nu^{\theta}_t},P_{\nu^{\theta}_t})$ in equation~\eqref{eq:diversed_b_p}.  \COMMENT{\textcolor{blue}{dynamic ensemble in Section~\ref{sec:diversed}}}
  
  \FOR{$t=n$ \TO $n+N$}
    \STATE Draw $u \sim \mathrm{Uniform}(0,1)$.
    \IF{$u \le b_{\nu^{\theta}_t}(\tilde x_t | x_{1:n-1},\,\tilde x_{n:t-1})$}
      \STATE \textbf{accept} $\tilde x_t$; 
      \STATE set $x_n \gets \tilde x_t$; 
      \STATE update $n \gets n+1$.
    \ELSE
      \STATE \textbf{reject} $\tilde x_t$; \COMMENT{discard $\tilde x_{t+1:T}$ and resume}
      \STATE sample $x_n \sim P_{\nu^{\theta}_t}$;
      \STATE update $n \gets n+1$; 
      \STATE \textbf{break} 
    \ENDIF
  \ENDFOR
\ENDWHILE
\end{algorithmic}
\end{algorithm}

\newpage

\section{Theoretical Justification}

\begin{table}[!ht]
    \centering
     \caption{List of notations and their definitions.}
    \label{tab:symbols}
    \begin{tabular}{r|l}
    \toprule
        \textbf{Notation} & \textbf{Definition} \\ 
    \midrule
        $x_0$ & Input prompt \\
         $N$ & Draft sequence length \\
        $x_{1:N}$ & Sampled accepted token sequence \\
        $T_N$ & Number of rejections when sampling a sequence of length $N$ \\
        \hline
        $M_q$ & Draft (small) model \\
        $M_p$ & Target (large) model \\
        \hline
        $q(\cdot | x_{1:t-1})$ or $q_t(\cdot)$ & Draft (small) model sampling distribution at time $t$. \\
        $p(\cdot | x_{1:t-1})$ or $p_t(\cdot)$ & Target (small) model sampling distribution at time $t$. \\
        \hline
        $\nu$  &  static  ensemble for between draft and target model.\\
        $\nu^\theta$ & Proposed  dynamic ensemble for between draft and target model.\\
    \bottomrule
    \end{tabular}
\end{table}

\subsection{Proof of Proposition \ref{prop:optimal}}
\label{appex:proof_optimal}

\begin{proof}
Recall that $\nu_t$ is defined in Equation~\eqref{eq:nu_t} as the static ensemble
\[
\nu_t(\cdot)= w\,p_t(\cdot) + (1-w)\,q_t(\cdot).
\]
The rejection probability under the acceptance rule $b_{\nu_t}(x)=\min\{1,\nu_t(x)/q_t(x)\}$ is
\begin{align*}
\PP_{\nu_t}(\text{reject}) = 1-\sum_x \min\left\{1, \frac{w p(x)+(1-w)q(x)}{q(x)}\right\}q(x) 
&= 1- \sum_x \min\left\{q(x), w p(x)+(1-w)q(x)\right\} \\
&= w - \sum_x \min\{w q(x), w p(x)\} \\
&= w \left(1- \sum_x\min\{q(x), p(x)\}\right) \\
&= w {\TV}(p, q). 
\end{align*}
Next, the total variation distance between $\nu$ and $p$ satisfies:
\begin{align*}
{\TV}(\nu_t, p) = \sum_x \max\{0, w p(x)+(1-w)q(x)-p(x)\} = \sum_x \max\{0, (1-w)(q(x)-p(x))\} 
&= (1-w) {\TV}(q, p) \\
&= (1-w) {\TV}(p, q). 
\end{align*}   
Combining the two identities gives
\[
\PP_{\nu_t}(\mathrm{reject}) + \TV(\nu_t,p_t) \;=\; \TV(q_t,p_t),
\]
which concludes the proof.
\end{proof}

\subsection{Proof of Lemma~\ref{lem:linear}} \label{lem:expect-proof}

\begin{proof}
Conditioned on $x_{0:t-1}$, the draft proposal $\tilde{x}_t$ is sampled from $q_t$.
The expected probability of accepting the $\tilde{x}_t$ given $x_{0:t-1}$ is:
\begin{align*}
    \E_{x\sim q_t}\left[\min\left\{1, \frac{\nu_t(x)}{q_t(x)}\right\}\right] = \sum_x\min\left\{1, \frac{\nu_t(x)}{q_t(x)}\right\} q_t(x) 
    & = \sum_x \min\left\{q_t(x), \nu_t(x)\right\} \\
\end{align*}
Using $\nu_t(x)=wp_t(x)+(1-w)q_t(x)$ as defined in Equation~\eqref{eq:nu_t}, we obtain
\begin{align*}
    \sum_x \min\{q_t(x),\nu_t(x)\} = \sum_x \min\left\{q_t(x), wp_t(x)+(1-w)q_t(x)\right\}  & = \sum_x q_t(x) + \sum_x \min\left\{0, w(p_t(x)-q_t(x))\right\} \\
    & = 1 - w + \sum_x w\min\left\{q_t(x), p_t(x)\right\} 
\end{align*}
Finally, recall the identity $\TV(q_t,p_t)=1-\sum_x \min\{q_t(x),p_t(x)\}$. Substituting yields
\begin{align*}
(1-w) + w\sum_x \min\{q_t(x),p_t(x)\} = 1 - w\Bigl(1-\sum_x \min\{q_t(x),p_t(x)\}\Bigr) = 1 - w\,\TV(q_t,p_t),
\end{align*}
as claimed.
\end{proof}

\subsection{Proof of Lemma \ref{lem:accept_drift}} \label{appex:proof_bayes}

\begin{proof}
Condition on the prefix $x_{0:t-1}$. By Bayes' rule,
\begin{align*}
\PP(x_t=x | x_t \text{ is accepted},x_{0:t-1})
&=
\frac{\PP(x_t \text{ is accepted} | x_t=x, x_{0:t-1})\,\PP(x_t=x\mid x_{0:t-1})}
{\sum_{x'} \PP(x_t \text{ is accepted} | x_t=x', x_{0:t-1})\,\PP(x_t=x'\mid x_{0:t-1})}.
\end{align*}
Under lossless speculative decoding, a draft proposal $x$ is accepted with probability
$\min\{1,\,\frac{p(x| x_{0:t-1})}{q(x| x_{0:t-1})}\}$, and $\PP(x_t=x\mid x_{0:t-1})=q(x| x_{0:t-1})$.
Substituting these expressions yields
\begin{align*}
\PP(x_t=x | x_t \text{ is accepted},x_{0:t-1})
&=
\frac{\min\!\left\{1,\frac{p(x| x_{0:t-1})}{q(x| x_{0:t-1})}\right\}q(x| x_{0:t-1})}
{\sum_{x'} \min\!\left\{1,\frac{p(x'| x_{0:t-1})}{q(x'| x_{0:t-1})}\right\}q(x'| x_{0:t-1})} =
\frac{\min\{q(x| x_{0:t-1}),\,p(x| x_{0:t-1})\}}
{\sum_{x'} \min\{q(x'| x_{0:t-1}),\,p(x'| x_{0:t-1})\}}.
\end{align*}
Finally, using a standard property of total variation distance
$\sum_{x'} \min\{q(x'| x_{0:t-1}),p(x'| x_{0:t-1})\}
= 1-\TV(q(\cdot| x_{0:t-1}),p(\cdot| x_{0:t-1}))$, we have:
\begin{align*}
    \PP(x_t=x | x_t \text{ is accepted},x_{0:t-1}) = \frac{\min\{q(x | x_{0:t-1}), p(x | x_{0:t-1})\}}{1-\TV(q(x | x_{0:t-1}),  p(x | x_{0:t-1}))}=
\tilde{m}(x \mid x_{0:t-1}).
\end{align*}
This completes the proof.
\end{proof}

\subsection{Proof of Theorem~\ref{theom:expected_accept}}\label{appex:proof_expected_accept}

\begin{proof}
Fix the draft length $N$,   recall that  $A_i$ denote the \emph{acceptance} indicator at step $i$.
$A_i=1$ if the $i$-th draft token $x_i$ is accepted given the prefix $x_{0:i-1}$, and
$A_i=0$ otherwise.

We start by computing the conditional expectation, given a draft generation ${x}_{1:N}$, the expected the acceptance length $T_N$ is
\begin{equation*}
    \E[T_N | {x}_{1:N}] = \sum_{t=1}^{N} \PP(T_N \ge t |  {x}_{1:N}) 
\end{equation*}
where the expectation of an indicator is the probability of its event. The  event $\{T_N\ge t\}$ means the first $t$ draft tokens are accepted, i.e., $\bigcap_{i=1}^{t}\{A_i=1\}$. Since the acceptance decisions across steps are \textit{independent} given the prefixes, the probability can be factorized as:
\begin{align}\label{eq:tail_product}
 \PP(T_N \ge t |  {x}_{1:N}) =  \prod_{i=1}^{t} \PP(A_i = 1 | {x}_{1:N}) = \prod_{i=1}^{t} \mint{i} 
\end{align}
where the acceptance probability $ \PP(A_i = 1 | {x}_{1:N})$ equals to the accept rule $b(x_i)$ under speculative decoding.

Taking the expectation over the draft sequence and using the tower law gives
\begin{align}
    \E[T_N] = \E_{x_{1:N}\sim q} \left[\E [T_N | x_{1:N}] \right]   & = \sum_{x_{1:N}}\left[\sum_{t=1}^{N} \prod_{i=1}^{t} \mint{i}\right] q(x_{1:N}) \nonumber\\
    & = \sum_{x_{1:N}}\left[\sum_{t=1}^{N} \left[\prod_{i=1}^{t} \mint{i}\right] \prod_{i=1}^N q(x_i | x_{0:i-1})\right] \nonumber\\
    & = \sum_{x_{1:N}}\left[\sum_{t=1}^{N} \prod_{i=1}^{t}\min\{q(x_i | x_{0:i-1}), p(x_i | x_{0:i-1})\} \prod_{i=t+1}^{N} q(x_i | x_{0:i-1})  \right] \nonumber\\
    & = \sum_{t=1}^{N} \left[\sum_{x_{1:N}}\prod_{i=1}^{t}\min\{q(x_i | x_{0:i-1}), p(x_i | x_{0:i-1})\} \prod_{i=t+1}^{N} q(x_i | x_{0:i-1})\right] \nonumber \\
    & = \sum_{t=1}^N\underbrace{ \sum_{x_{1:t}} \prod_{i=1}^{t}\min\{q(x_i | x_{0:i-1}), p(x_i | x_{0:i-1})\}}_{\text{denote this part as } S_t}. \label{eq:intermediate_1} 
\end{align}
where we used $q(x_{1:N})=\prod_{i=1}^N q(x_i\mid x_{0:i-1})$ and
$\sum_{x_{t+1:N}}\prod_{i=t+1}^{N} q(x_i\mid x_{0:i-1})=1$.

We expand $S_t$ by summing out the last token $x_t$ first. Since the product separates into
a factor that depends on $x_t$ and a factor that does not, we have
\begin{align*}
S_t
&= \sum_{x_{1:t-1}}
\left(\prod_{i=1}^{t-1}\min\{q(x_i\mid x_{0:i-1}),p(x_i\mid x_{0:i-1})\}\right)
\sum_{x_t}\min\{q(x_t\mid x_{0:t-1}),p(x_t\mid x_{0:t-1})\}.
\end{align*}
Next, we use a identity of TV distance for discrete distributions $\sum_{x}\min\{q(x),p(x)\} = 1-\TV(q,p)$, applied conditionally given the prefix $x_{0:t-1}$. This yields
\begin{align}
S_t
&= \sum_{x_{1:t-1}}
\left(\prod_{i=1}^{t-1}\min\{q(x_i\mid x_{0:i-1}),p(x_i\mid x_{0:i-1})\}\right)
\Bigl(1-\TV(q(\cdot\mid x_{0:t-1}),p(\cdot\mid x_{0:t-1}))\Bigr).
\label{eq:St_one_step}
\end{align}

Now we convert the remaining min-products into expectations under the min-kernel
$\tilde m$ from Lemma~\ref{lem:accept_drift}. We expand the defined definition of $\tilde m$ as:
\[
\min\{q(x_i\mid x_{0:i-1}),p(x_i\mid x_{0:i-1})\}
=
\Bigl(1-\TV(q(\cdot\mid x_{0:i-1}),p(\cdot\mid x_{0:i-1}))\Bigr)\,
\tilde m(x_i\mid x_{0:i-1}).
\]
Substituting this identity into the product in \eqref{eq:St_one_step} for each
$i=1,\dots,t-1$ gives
\begin{align*}
S_t
&=
\sum_{x_{1:t-1}}
\left(\prod_{i=1}^{t-1}
\Bigl(1-\TV(q(\cdot\mid x_{0:i-1}),p(\cdot\mid x_{0:i-1}))\Bigr)\,
\tilde m(x_i\mid x_{0:i-1})\right)
\Bigl(1-\TV(q(\cdot\mid x_{0:t-1}),p(\cdot\mid x_{0:t-1}))\Bigr) \\
&=
\sum_{x_{1:t-1}}
\left(\prod_{i=1}^{t-1}\tilde m(x_i\mid x_{0:i-1})\right)
\left(\prod_{i=0}^{t}
\Bigl(1-\TV(q(\cdot\mid x_{0:i}),p(\cdot\mid x_{0:i}))\Bigr)\right),
\end{align*}

Finally, the sum over $x_{1:t-1}$ is exactly expectation with respect to the
sequential measure induced by $\tilde m$, hence
\[
S_t
=
\E_{x_{1:t-1}\sim \tilde m}\left[
\prod_{i=0}^{t}\Bigl(1-\TV(q(\cdot\mid x_{0:i}),p(\cdot\mid x_{0:i}))\Bigr)
\right].
\]

Combined with Equation~\eqref{eq:intermediate_1}, we finish the proof. 
\end{proof}

\newpage
\section{Extended  Experimental Analysis} \label{appx:exp-set}

\subsection{Implementation Details}

\paragraph{Implementation Detail of DIVERSED.}
The neural network architecture of the proposed \method consists of a single fully connected layer, whose input and output dimensions match the input vector. The activation function used is the sigmoid function.

We train separate ensembles per task.
For GSM8K and MBPP, the outcome reward is task accuracy; for CNNDM and XSum, we use ROUGE-2 \citep{lin2004rouge} following \cite{zhang2018minimum}.
We test temperatures \(0\) (greedy) and \(1\), and draft lengths \(N \in \{3,5,7\}\).
Generation budgets are 128 tokens for CNNDM and XSum, 384 for GSM8K, and 512 for MBPP.
We use few-shot prompting by supplying input–output exemplars that enable the LLM to infer the desired generation style.
All training and inference run on \(8\times\) A100-40GB GPUs.

\begin{table}[!ht]
    \centering
     \caption{Hyper-parameter configuration for each benchmark dataset.}
    \label{tab:param-config}
    \begin{tabular}{r|cccc}
    \toprule
         & CNNDM & XSUM & MBPP & GSM8K \\
    \midrule
        Generation budget & 128 & 128 & 384 & 512 \\
        Hyper Parameter $\lambda$ & 0.5 & 0.5 & 0.3 & 0.1 \\
    Draft length $N$ & \multicolumn{4}{c}{$\{3,5,7\}$} \\
    Rollout Number  & \multicolumn{4}{c}{4}\\
    Learning Rate & \multicolumn{4}{c}{5e-4} \\
    Training Epochs & \multicolumn{4}{c}{2} \\
    \midrule
    \multicolumn{5}{c}{\textbf{(a)} Target/Draft model pair is {Llama3.1-8B-Instruct}/{Llama3.2-1B-Instruct}.}\\
    \midrule
        Generation budget & 128 & 128 & 384 & 512 \\
        Hyper Parameter $\lambda$ & 0.5 & 0.5 & 0.3 & 0.3 \\
    Draft length $N$ & \multicolumn{4}{c}{$\{3,5,7\}$} \\
    Rollout Number  & \multicolumn{4}{c}{4}\\
    Learning Rate & \multicolumn{4}{c}{5e-4} \\
    Training Epochs & \multicolumn{4}{c}{2} \\
    \midrule
    \multicolumn{5}{c}{\textbf{(b)} Target/Draft model pair is {Qwen3-8B}/{Qwen3-0.6B}.}\\
        \midrule
        Generation budget & 128 & 128 & 384 & 512 \\
        Hyper Parameter $\lambda$ & 0.5 & 0.5 & 0.3 & 0.35 \\
    Draft length $N$ & \multicolumn{4}{c}{$\{3,5,7\}$} \\
    Rollout Number  & \multicolumn{4}{c}{4}\\
    Learning Rate & \multicolumn{4}{c}{5e-4} \\
    Training Epochs & \multicolumn{4}{c}{2} \\
    \midrule
    \multicolumn{5}{c}{\textbf{(c)} Target/Draft model pair is {gemma-3-12b-it}/{gemma-3-4b-it}.}
    \end{tabular}
\end{table}

In terms of baseline selection, we omit Medusa~\citep{Cai2024Medusa} and EAGLE~\citep{li2024eagle} because they modify either the target LLM or the draft LLM, whereas \method keeps both models unchanged. A direct comparison would therefore conflate improvements from verification with changes in model architecture, making the comparison less controlled and less fair.

We view \method as complementary to DISCO, Medusa, and EAGLE. In principle, \method could be layered on top of these approaches by replacing their rigid verification step with our relaxed, ensemble-based verifier, potentially yielding additional efficiency gains. However, due to engineering complexity and the limited rebuttal timeline, we did not evaluate these combinations. All baselines we report follow the same line of work that holds the draft and target models fixed and differs only in how the acceptance rule is defined.

\newpage
\subsection{Extended Experiment Results}\label{appx:more-exp}

\begin{table*}[!ht]
\centering
\caption{Reported acceptance rate and generation quality. Columns are grouped by task. Each block compares three baselines and our method for a given \textit{Target}/\textit{Draft} pair. All results use temperature 0 and draft length $N=5$. 
Our \method enjoys a higher acceptance rate and also maintains comparable quality to baselines.}
\label{tab:main_acceptancerate_false}
\resizebox{\linewidth}{!}{%
\begin{tabular}{l|lc|cc|cc|cc}
\toprule
&
\multicolumn{2}{c}{\textbf{GSM8K}} &
\multicolumn{2}{|c}{\textbf{CNNDM}} &
\multicolumn{2}{c|}{\textbf{Xsum}} &
\multicolumn{2}{|c}{\textbf{MBPP}} \\
\cmidrule(lr){2-3}\cmidrule(lr){4-5}\cmidrule(lr){6-7}\cmidrule(lr){8-9}

& \makecell{Accept \\Rate ($\uparrow$)} & \makecell{Quality\\Accuracy\,($\uparrow$)} &
\makecell{Accept \\Rate($\uparrow$)} & \makecell{Quality\\ROUGE-2\,($\uparrow$)} &
\makecell{Accept \\Rate($\uparrow$)} & \makecell{Quality\\ROUGE-2\,($\uparrow$)} &
\makecell{Accept \\Rate($\uparrow$)} & \makecell{Quality\\pass@1\,($\uparrow$)} \\
\midrule
Autoregressive&NA & 80$\%$ &NA & 11.29 &NA & 8.42 &NA & 62$\%$\\
\hline
SD & 61.53$\%$ & 80$\%$ & 71.70$\%$ & 11.26 & 50.24$\%$ & 8.36 & 72.18$\%$ & 62$\%$ \\
Static Ensemble & 82.58$\%$ & 79$\%$ & \textbf{85.04$\%$} & 11.45 & 61.44$\%$ & 7.53 & 84.67$\%$ & 61$\%$ \\
{\method} (ours)  & \textbf{84.82$\%$} & 80$\%$ & {81.92$\%$} & \textbf{12.37} & \textbf{94.81$\%$} & {7.66} & \textbf{85.03$\%$} & {62$\%$} \\
\midrule
  \multicolumn{9}{c}{ \textbf{(a)} Target/Draft model pair is {Llama-3.1-8B}/{Llama-3.2-1B}.}\\
\midrule
Autoregressive&NA & 92$\%$ &NA & 11.48 &NA & 5.38 &NA & 57$\%$ \\
\hline
SD & 63.70$\%$ & 92$\%$ & 36.40$\%$ & 11.38 & 19.56$\%$ & 5.35 & 70.27$\%$ & 57$\%$ \\
Static Ensemble  & 67.58$\%$ & 92$\%$ & 41.60$\%$ & 11.85 & 34.14$\%$ & 5.39 & \textbf{75.70$\%$} & 57$\%$ \\
{\method} (ours)  & \textbf{69.68$\%$} & 93$\%$ & \textbf{58.46$\%$} & 12.64 & \textbf{38.12$\%$} & 5.54 & 74.83$\%$ & \textbf{59$\%$} \\
\midrule
    \multicolumn{9}{c}{ \textbf{(b)} Target/Draft model pair is {Qwen3-8B}/{Qwen3-0.6B}.} \\
\midrule
Autoregressive&NA & 93$\%$ &NA & 10.87 &NA & 9.16 &NA & 66$\%$ \\
\hline
SD & 84.15$\%$ & 92$\%$ & 40.39$\%$ & 10.62 & 35.54$\%$ & 9.04 & 83.89$\%$ & 66$\%$ \\
Static Ensemble & 91.66$\%$ & 92$\%$ & 54.48$\%$ & 10.84 & 56.31$\%$ & 7.90 & 86.69$\%$ & 65$\%$ \\
{\method} (ours) & \textbf{92.35$\%$} & 93$\%$ & \textbf{66.39$\%$} & 11.47 & \textbf{63.38$\%$} & {7.22} & \textbf{90.75$\%$} & {64$\%$} \\
\bottomrule
\multicolumn{9}{c}{ \textbf{(c)} Target/Draft model pair is {Gemma-3-12B}/{Gemma-3-4B}.}\\
\end{tabular}%
}
\end{table*}

\begin{table*}[!ht]
\centering
\caption{Reported acceptance rate and generation quality, under different temperature. Columns are grouped by task. Each block compares three baselines and our method for a given \textit{Target}/\textit{Draft} pair. All results use a temperature of 0.1 and a draft length $N=5$. 
Our \method enjoys a higher acceptance rate and also maintains comparable quality to baselines.
}
\label{tab:main_acceptancerate_small}
\resizebox{\linewidth}{!}{%
\begin{tabular}{l|lc|cc|cc|cc}
\toprule

&\multicolumn{2}{c}{\textbf{GSM8K}} &
\multicolumn{2}{|c}{\textbf{CNNDM}} &
\multicolumn{2}{|c}{\textbf{XSum}} &
\multicolumn{2}{|c}{\textbf{MBPP}} \\
\cmidrule(lr){2-3}\cmidrule(lr){4-5}\cmidrule(lr){6-7}\cmidrule(lr){8-9}
\textbf{Method} &
\makecell{Accept \\Rate ($\uparrow$)} & \makecell{Quality\\Accuracy\,($\uparrow$)} &
\makecell{Accept \\Rate($\uparrow$)} & \makecell{Quality\\ROUGE-2\,($\uparrow$)} &
\makecell{Accept \\Rate($\uparrow$)} & \makecell{Quality\\ROUGE-2\,($\uparrow$)} &
\makecell{Accept \\Rate($\uparrow$)} & \makecell{Quality\\pass@1\,($\uparrow$)} \\
\midrule
Autoregressive & NA & 76$\%$ & NA & 9.94 & NA & 7.56 & NA & 62$\%$\\
\hline
SD & 66.41$\%$ & 76$\%$ & 17.69$\%$ & 9.92 & 37.67$\%$ & 7.82 & 72.94$\%$ & 61$\%$ \\
SD (Lossy) & 67.42$\%$ & 76$\%$ & 17.85$\%$ & 10.17 & 38.83$\%$ & 7.97 & 73.37$\%$ & 61$\%$ \\
SpecCascade & 68.01$\%$ & 76$\%$ & 52.20$\%$ & 10.19 & 40.22$\%$ & 7.53 & 84.15$\%$ & 61$\%$ \\
Static Ensemble & 72.36$\%$ & 75$\%$ & 87.58$\%$ & 10.46 & 51.20$\%$ & 7.31 & 80.66$\%$ & 61$\%$ \\
{\method} (ours)  & \textbf{89.14$\%$} & 76$\%$ & \textbf{92.11$\%$} & {11.65} & \textbf{64.75$\%$} & {7.52} & \textbf{89.78$\%$} & {62$\%$} \\
\bottomrule
\multicolumn{9}{c}{ \textbf{(a)} Target/Draft model pair is {Llama-3.1-8B}/{Llama-3.2-1B}.}\\
\midrule

Autoregressive & NA & 92$\%$ & NA & 11.02 & NA & 5.21 & NA & 57$\%$ \\
\hline
SD & 58.83$\%$ & 92$\%$ & 34.19$\%$ & 10.35 & 16.45$\%$ & 5.26 & 69.49$\%$ & 56$\%$ \\
SD (lossy) & 59.85$\%$ & 92$\%$ & 35.32$\%$ & 10.74 & 19.56$\%$ & 4.90 & 69.71$\%$ & 56$\%$ \\
 SpecCascade & 59.71$\%$ & 92$\%$ & 37.67$\%$ & 11.30 & 32.98$\%$ & 4.91 & 78.32$\%$ & 56$\%$ \\
Static Ensemble & 88.75$\%$ & 92$\%$ & 40.67$\%$ & 10.90 & 30.80$\%$ & 5.19 & \textbf{79.86$\%$} & 55$\%$ \\
{\method} (ours)  & \textbf{91.84$\%$} & 92$\%$ & \textbf{46.40$\%$} & 11.37 & \textbf{45.87$\%$} & 5.25 & {75.56$\%$} & 59$\%$ \\
\bottomrule
  \multicolumn{9}{c}{ \textbf{(b)} Target/Draft model pair is {Qwen3-8B}/{Qwen3-0.6B}.}  \\
\midrule
Autoregressive & NA & 93$\%$ & NA & 9.01 & NA & 8.27 & NA & 66$\%$ \\
\hline
SD & 81.85$\%$ & 92$\%$ & 44.27$\%$ & 9.00 & 39.63$\%$ & 8.27 & 83.85$\%$ & 67$\%$ \\
SD (Lossy) & 82.75$\%$ & 92$\%$ & 48.73$\%$ & 10.71 & 40.08$\%$ & 8.37 & 84.31$\%$ & 66$\%$ \\
SpecCascade & 83.50$\%$ & 93$\%$ & 48.72$\%$ & 10.72 & 39.42$\%$ & 8.79 & 84.68$\%$ & 67$\%$ \\
Static Ensemble & 89.67$\%$ & 91$\%$ & 47.47$\%$ & 10.74 & 62.65$\%$ & 8.31 & 88.34$\%$ & 66$\%$ \\
{\method} (ours) & \textbf{90.26$\%$} & {92$\%$} & \textbf{64.26$\%$} & 11.60& \textbf{71.22$\%$} & {7.84} & \textbf{91.23$\%$} & {67$\%$} \\
\bottomrule
\multicolumn{9}{c}{ \textbf{(c)} Target/Draft model pair is {Gemma-3-12B}/{Gemma-3-4B}.}
\end{tabular}%
}
\end{table*}

\newpage

\begin{figure}[!hbtp]
\centering
\includegraphics[width=0.6 \linewidth]{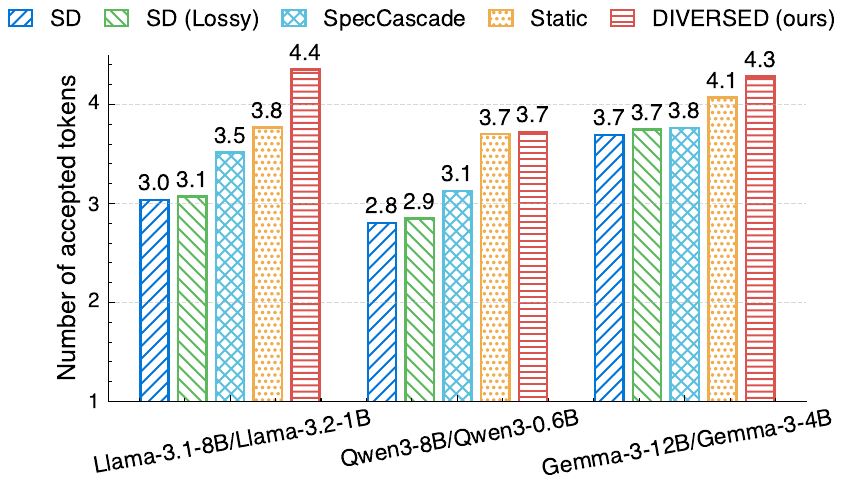} 
\caption{Average number of tokens accepted per verification round (temperature $=0.1$, draft length $N=5$) across three target/draft model pairs. Our \method consistently surpasses all baselines by achieving a higher number of accepted tokens.}
\label{fig:token_speed_up_temp_small}
\end{figure}

\begin{figure*}[!hbtp]
	\centering
    \includegraphics[width=\textwidth]{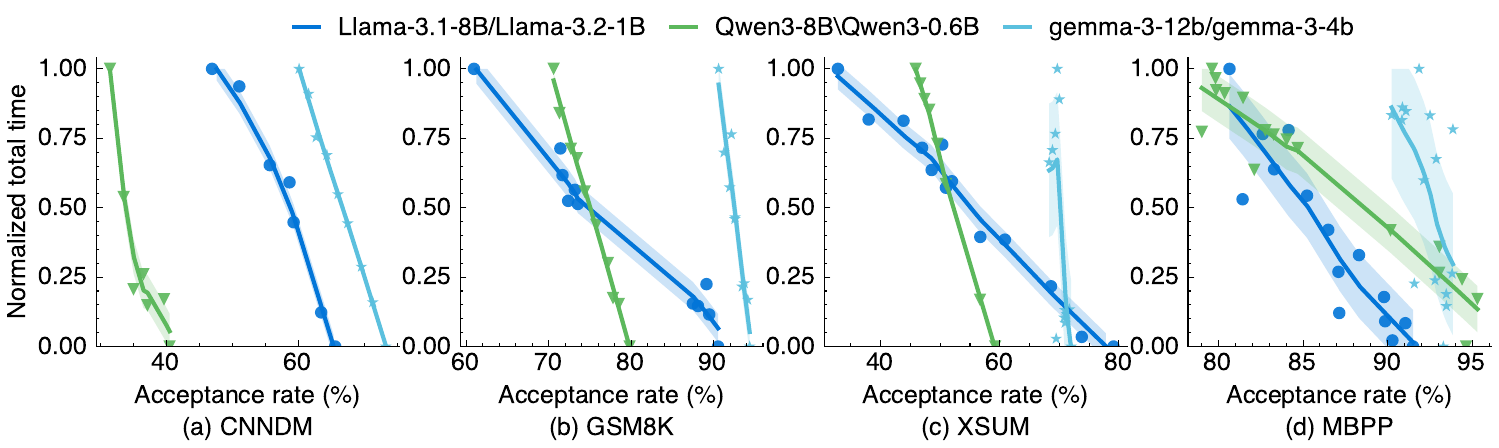}
    \vspace{-3mm}
    \caption{Results on the CNNDM, XSUM, MBPP, and GSM8k datasets for three target/draft pairs: normalized total time (lower is better) versus acceptance rate (higher is better) under speculative decoding with $N{=}3$. Model pairs with higher acceptance achieve lower decoding time.}
\label{fig:acceptance_total_time_len3}
\end{figure*}

\begin{figure*}[!hbtp]
	\centering
    \includegraphics[width=\textwidth]{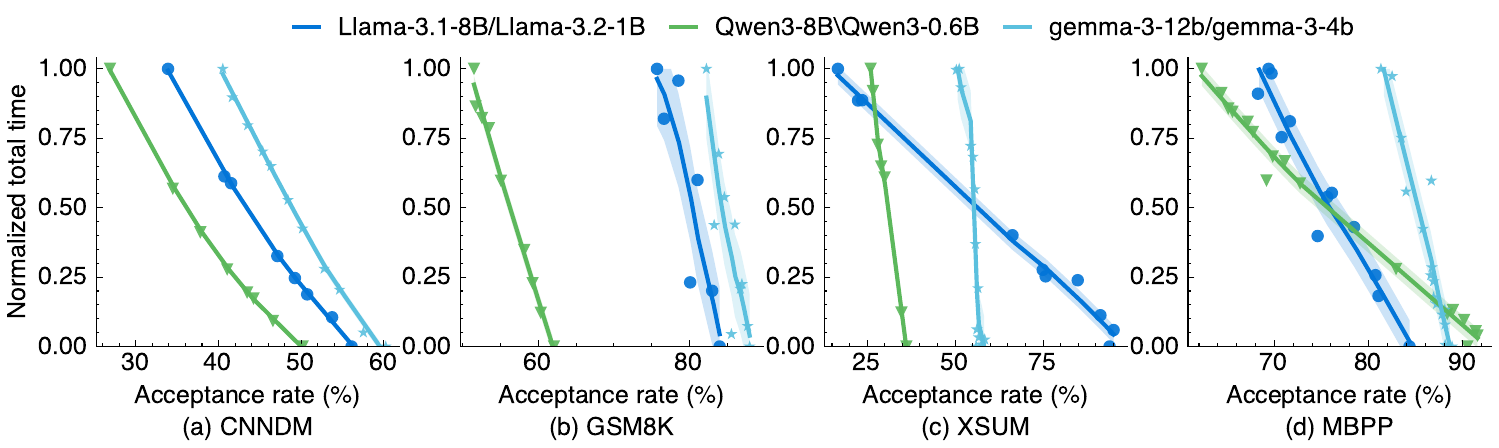}
    \caption{Results on the CNNDM, XSUM, MBPP, and GSM8k datasets for three target/draft pairs: normalized total time (lower is better) versus acceptance rate (higher is better) under speculative decoding with $N{=}7$. Model pairs with higher acceptance achieve lower decoding time.}
\label{fig:acceptance_total_time_len7}
\end{figure*}

\begin{figure*}[!hbtp]
	\centering
    \includegraphics[width=\textwidth]{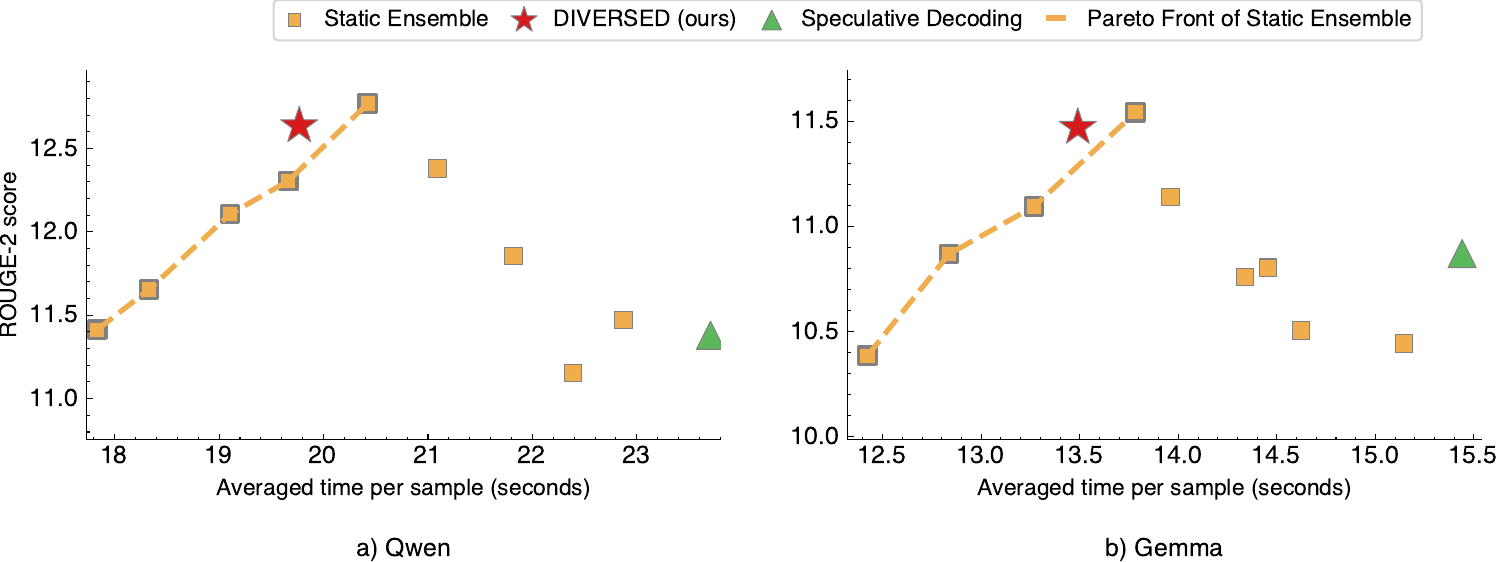}
    \caption{Our \method outperforms the Pareto front of the static ensemble and other baselines. 
	Results are reported on the CNNDM dataset using the target/draft model pair: Qwen3-8B/Qwen3-0.6B and Gemma-3-12b-it/Gemma-3-4b-it, with temperature set to $0$.}
\label{fig:rouge_time_qwen_gemma}
\end{figure*}

\newpage
\clearpage
\subsection{GSM8K Case Study}

\SampleBoxText{Case Study on the GSM8K dataset using Llama model.}{Question: Marissa is hiking a 12-mile trail. She took 1 hour to walk the first 4 miles, then another hour to walk the next two miles. If she wants her average speed to be 4 miles per hour, what speed (in miles per hour) does she need to walk the remaining distance?}

\begin{tabularx}{\textwidth}{@{} l X @{}}
	\toprule

	Autoregressive &  Marissa walks the remaining 6 miles at her desired average speed of 4 miles per hour, so the time it will take her to walk the remaining distance is 6 / 4 = 1.5 hours.
    
    So in total she will have spent 1.5 + 1 + 1 = 3.5 hours walking the 12-mile trail.
    
    $\#\#\# 6$

 \\
	\midrule
SD &   She has walked 4 + 2 = 6 miles in 2 hours.

So, she has 12 - 6 = 6 miles left to walk.

She wants to walk the entire trail in 12 / 4 = 3 hours.

So, she needs to walk the remaining 6 miles in 3 - 2 = 1 hour.

Therefore, she needs to walk the remaining 6 miles in 6 / 1 = 6 miles per hour.

$\#\#\# 6$
 \\
	\midrule
	SD (Lossy) & She has walked 4 + 2 = 6 miles in 2 hours.
    
    So, she has 12 - 6 = 6 miles left to walk.
    
    She wants her average speed to be 4 miles per hour, so she needs to walk the remaining 6 miles in 6 / 4 = 1.5 hours.
    
    Therefore, she needs to walk the remaining distance at a speed of 6 / 1.5 = 4 miles per hour.
    
   \textcolor{blue}{ $\#\#\# 4$}

 \\
	\midrule
SpecCascade &   First, let's find out how long it will take her to walk the entire trail. We know that she will need to walk 12 miles at a speed of 4 miles per hour, so the time it will take her to walk the entire trail is 12 / 4 = 3 hours. 

She already walked 6 miles and it took her 2 hours to do so, so the time she has left is 3 - 2 = 1 hour. 

Therefore, she needs to walk 12 - 6 = 6 miles in 1 hour, so she needs to walk at a speed of 6 / 1 = 6 miles per hour.

$\#\#\# 6$

 \\
	\midrule
	Static Ensemble &   She has walked 4 miles in 1 hour, and 2 miles in another hour, so the total distance covered is 6 miles in 2 hours, which means her average speed is 6 / 2 = 3 miles per hour. 
    
    To cover the remaining 6 miles, she needs to cover 3 miles per hour, which means 6 miles in 2 hours. Therefore, her average speed must be 6 / 2 = 3 miles per hour.
    
    \textcolor{blue}{$\#\#\# 3$.}
     \\
	\midrule
	\method &   

    She has walked 4 + 2 = 6 miles in 1 + 1 = 2 hours.
    
    She has 12 - 6 = 6 miles left to walk.
    
    She wants to walk the entire trail in 12 / 4 = 3 hours.
    
    So she has 3 - 2 = 1 hour to walk the remaining 6 miles.
    
    Therefore, she needs to walk 6 / 1 = 6 miles per hour.

    $\#\#\# 6$

\\
	\midrule
	\midrule

\end{tabularx}

\newpage
\SampleBoxText{Case Study on the GSM8K dataset using Qwen model.}{Question: Marissa is hiking a 12-mile trail. She took 1 hour to walk the first 4 miles, then another hour to walk the next two miles. If she wants her average speed to be 4 miles per hour, what speed (in miles per hour) does she need to walk the remaining distance?} 

\begin{tabularx}{\textwidth}{@{} l X @{}}
	\toprule
	Autoregressive & \textcolor{orange}{To calculate the total time} she needs to complete the trail for an average speed of 4 mph, we calculate the total time as total distance divided by average speed, which is 12 miles / 4 mph = 3 hours.
    
    So far, she has spent 1 + 1 = 2 hours. This leaves her with 3 - 2 = 1 hour to walk the remaining distance.
    
    She has 12 - 4 - 2 = 6 miles left. Therefore, she needs to walk 6 miles in 1 hour, which is a speed of 6 mph.
    
    $\#\#\# 6$

 \\
	\midrule
SD &   \textcolor{green}{The total distance of the trail is 12 miles.} So far, she has walked 4 + 2 = 6 miles. Therefore, she has 12 - 6 = 6 miles left. She wants her average speed to be 4 miles per hour for the entire trail. To find the total time she has for the whole trail, we can use the formula: total time = total distance / average speed = 12 miles / 4 mph = 3 hours. She has already spent 1 + 1 = 2 hours. Therefore, she has 3 - 2 = 1 hour left to walk the remaining 6 miles. To find the required speed, we can use the formula: speed = distance / time = 6 miles / 1 hour = 6 miles per hour. So, she needs to walk the remaining distance at 6 miles per hour.

$\#\#\# 6$
 \\
	\midrule
	SD (Lossy) & \textcolor{blue}{To find the average speed}, we need to divide the total distance by the total time. Marissa has already walked 4 + 2 = 6 miles. She has 12 - 6 = 6 miles left. She has already spent 1 + 1 = 2 hours. To have an average speed of 4 miles per hour, she needs to complete the 12-mile trail in 12 / 4 = 3 hours. She has 3 - 2 = 1 hour left to walk the remaining 6 miles. Therefore, she needs to walk at a speed of 6.

    $\#\#\# 6$

 \\
	\midrule
SpecCascade &   \textcolor{blue}{To find the average speed}, we need to divide the total distance by the total time. Marissa has already walked 4 + 2 = 6 miles. She has 12 - 6 = 6 miles left. She has already spent 1 + 1 = 2 hours. To have an average speed of 4 miles per hour, the total time for the entire trail should be 12 / 4 = 3 hours. Therefore, she has 3 - 2 = 1 hour left to walk the remaining 6 miles. So, she needs to walk at a speed of 6.

$\#\#\# 6$

 \\
	\midrule
	Static Ensemble &  \textcolor{blue}{To find the average speed}, we need to divide the total distance by the total time. Marissa has already walked 4 + 2 = 6 miles. She has 12 - 6 = 6 miles left. She has already spent 1 + 1 = 2 hours. She wants her average speed to be 4 miles per hour, so the total time for the entire trail should be 12 / 4 = 3 hours. Therefore, she has 3 - 2 = 1 hour left to walk the remaining 6 miles. So, she needs to walk at a speed of 6 / 1 = 6 miles per hour.

    $\#\#\# 6$

     \\
	\midrule
	\method &   

    \textcolor{blue}{To find the average speed}, we need to divide the total distance by the total time. The total distance is 12 miles. The total time is 2 hours for the first 6 miles. She wants her average speed to be 4 miles per hour, so the total time should be 12 / 4 = 3 hours. Therefore, she needs to walk the remaining 6 miles in 3 - 2 = 1 hour.  So her speed for the remaining distance should be 6 / 1 = 6 miles per hour.

    $\#\#\# 6$

\\
	\midrule
	\midrule

\end{tabularx}

\newpage
\SampleBoxText{Case Study on the GSM8K dataset using Gemma model.}{Question: Marissa is hiking a 12-mile trail. She took 1 hour to walk the first 4 miles, then another hour to walk the next two miles. If she wants her average speed to be 4 miles per hour, what speed (in miles per hour) does she need to walk the remaining distance?} 

\begin{tabularx}{\textwidth}{@{} l X @{}}
	\toprule
	Autoregressive & The remaining distance is 12 - 4 - 2 = 6 miles.
    
    The total time she wants to take for the whole trail is 12 / 4 = 3 hours.
    
    She has already spent 1 + 1 = 2 hours hiking, so she has 3 - 2 = 1 hour left.
    
    The required speed for the last 6 miles in 1 hour is 6 / 1 = 6 miles per hour.
    
    $\#\#\# 6$

 \\
	\midrule
SD &  Total distance is 12 miles, and she wants her average speed to be 4 miles per hour.

So, the total time she can spend hiking is 12/4 = 3 hours.

She has already spent 1 + 1 = 2 hours.

Therefore, she can hike the remaining distance for 3 - 2 = 1 hour.

The remaining distance is 12 - 4 - 2 = 6 miles.

So, she needs to walk 6 miles in 1 hour, which means she needs to walk at a speed of 6 miles per hour.

$\#\#\# 6$

 \\
	\midrule
	SD (Lossy) & The total distance of the trail is 12 miles.
    
    Marissa has already walked 4 + 2 = 6 miles.
    
    The remaining distance is 12 - 6 = 6 miles.
    
    She wants her average speed to be 4 miles per hour for the entire 12-mile trail.
    
    So, the total time she should take to complete the 12-mile trail is 12 / 4 = 3 hours.
    
    She has already taken 1 + 1 = 2 hours.
    
    The remaining time she has is 3 - 2 = 1 hour.
    
    To walk the remaining 6 miles in 1 hour, she needs to walk at a speed of 6 / 1 = 6 miles per hour.

    $\#\#\# 6$

 \\
	\midrule
SpecCascade &  Total distance is 12 miles.
Distance covered is 4 + 2 = 6 miles.
Remaining distance is 12 - 6 = 6 miles.
She wants her average speed to be 4 miles per hour over the entire 12-mile trail.

Total time to complete the trail should be 12/4 = 3 hours.
Time spent so far is 1 + 1 = 2 hours.
Remaining time is 3 - 2 = 1 hour.
So she needs to walk the remaining 6 miles in 1 hour.
Therefore, the required speed is 6/1 = 6 miles per hour.

$\#\#\# 6$

 \\
	\midrule
	Static Ensemble & Total distance is 12 miles. Marissa has already walked 4 + 2 = 6 miles. The remaining distance is 12 - 6 = 6 miles.
    
    She wants her average speed to be 4 miles per hour over the entire 12-mile trail. So the total time she should take to walk the entire trail is 12 / 4 = 3 hours.
    
    She has already walked for 1 + 1 = 2 hours. Therefore, she has 3 - 2 = 1 hour to walk the remaining 6 miles.
    
    The speed she needs to walk the remaining distance is 6 miles / 1 hour = 6 miles per hour.

    $\#\#\# 6$

     \\
	\midrule
	\method &   

        The total distance she has to cover is 12 miles. She already covered 4 + 2 = 6 miles. She needs to cover 12 - 6 = 6 miles. So she needs to cover the remaining 6 miles in the remaining time. To find her average speed for the entire hike, she needs to calculate the total time taken for the entire hike and then determine how much time she has left to cover the remaining 6 miles.
        Her average speed is 4 miles per hour for the entire hike. So her total time should be 12 miles / 4 miles per hour = 3 hours.
        She has already taken 1 + 1 = 2 hours. Therefore, she has 3 - 2 = 1 hour left to cover the remaining 6 miles. To find the speed she needs to walk the remaining 6 miles, we can calculate 6 miles / 1 hour = 6 miles per hour.

        $\#\#\# 6$

\\
	\midrule
	\midrule

\end{tabularx}

\newpage

\subsection{CNNDM Case Study}

\SampleBoxText{Case Study on the CNNDM dataset using Llama model.}{\small News:     President Barack Obama held up the shrinking Florida Everglades today as proof positive that climate change is real and is threatening America's 'national treasures,' as well as the economies of their surrounding towns, which rely heavily on tourist dollars. 'Climate change can no longer be denied. It can't be edited out,' he said this afternoon in a speech at Everglades National Park commemorating Earth Day. 'It can't be omitted from the conversation. And action can no longer be delayed.' 'That's why I've committed the United States to lead the world in combatting this threat,' he said, throwing a spotlight on his administration's efforts to reduce the rate of global warming through the reduction of carbon emissions. The day trip to South Florida, on which Obama was accompanied by Bill Nye 'the Science Guy,' also highlighted the massive amount of fuel it takes to power Air Force One, however, 9,000 gallons, round trip. Scroll down for video. President Barack Obama takes a walking tour of the Anhinga Trail at Everglades National Park. Obama visited the subtropical swamps of the park today as part of a push to get Americans thinking and talking about the damage climate change is causing close to home . Obama held up the shrinking Florida Everglades today as proof positive that climate change is real and is threatening America's 'national treasures,' as well as the economies of their surrounding towns, which rely heavily on tourist dollars . 'Climate change can no longer be denied. It can't be edited out,' he said this afternoon in a speech at Everglades National Park commemorating Earth Day. 'It can't be omitted from the conversation. And action can no longer be delayed' Vehicles in the president's motorcade are seen in the Everglades National Park where Obama stopped for a visit on the Anhinga trail in Homestead. Asked earlier this week if Obama's gas-guzzling trip to Florida would undermine his Earth Day message, the White House said it would not and reaffirmed Obama's efforts to reduce pollutants . On of the helicopters escorting Marine One with President Barack Obama onboard, prepares to land at Miami's International Airport . I love the smell of jet fuel, Nye told reporters as he waited to board the president's 747 this morning for the flight to Florida, during which the White House said he would film a video with Obama. Asked earlier this week if Obama's gas-guzzling trip to Florida would undermine his Earth Day message, the White House said it would not and reaffirmed Obama's efforts to reduce pollutants. Obama took a half -hour walking tour of the the Anhinga Trail at the 1.5-million-acre national park today before making a plea for Americans to save the Everglades and other natural gems by making environmentally-conscious decisions. 'I can't think of a better way to spend Earth Day than in one of our nation's greatest natural treasures, the Everglades,' he told his audience, calling the swamp, which he acknowledged is not technically a swamp, 'magical.' 'Climate change is threatening this treasure and the communities that depend on it,' Obama warned. 'If we don't act,there may not be an Everglades as we know it.' The president said the fallout from rising temperatures is 'not a problem for another generation.' At least, 'not anymore,' he said. 'This is a problem now. It has serious implications for the way we live right now.' Obama said the rising sea level in South Florida and inward flow of salt water is evidence that climate change is already having a negative effect on communities. 'This is not some impossible problem that we cannot solve. We can solve it if we've got some political will,' he said. 'And we can solve it in a way that creates jobs. We can solve it in a way that doesn't disrupt our economy but enhances our economy. And it's a bipartisan issue.' 'That's why I've committed the United States to lead the world in combatting this threat,' he said, throwing a spotlight on his administration's efforts to reduce the rate of global warming through the reduction of carbon emissions . An alligator is spotted sunning itself along the Anhinga Trail at Everglades National Park as Obama took a tour . Wiping sweat from his brow as he braved the subtropical temperatures to make his mark on Earth Day, the president mingled with crowd-goers at the mid-day event before returning to Washington on Air Force One. Previewing the day trip, the White House had reiterated  Obama's belief that 'no challenge poses a greater threat to future generations than climate change.' It's statement pointed out that last year was the  warmest year on record and that 14 of the 15 of earth's hottest years happened this century. It said, 'Climate change is also affecting some of the most iconic places in our country, from disappearing glaciers in Glacier National Park to dying Joshua Trees in Joshua Tree National Park.' 'Recognized worldwide as a unique and treasured landscape, the Everglades is a perfect example of the threat we face from climate change, including rising sea levels that result in shoreline erosion and increased flooding.' Obama said it is this generation's 'turn to ensure that' natural monuments remain the 'birthright of all Americans for generations to come.' Wiping sweat from his brow as he braved the subtropical temperatures to make his mark on Earth Day, the president mingled with crowd-goers at the mid-day event before returning to Washington on Air Force One . The White House also asserted that climate change was already displacing animals and plants, including tropical orchids, some of which are only found in south Florida.' 'We are blessed with the most beautiful God-given landscape in the world,' the president said today. 'It's an incredible bounty that's been given to us. But we've got to be good stewards for it. We have to take care of it.' Obama said it is this generation's 'turn to ensure that' natural monuments remain the 'birthright of all Americans for generations to come.' 'So many people here are active in your communities, doing what's needed. The young people who are here, the next generation, they're way ahead of us in understanding how important this is . 'Let's make sure we don't disappoint them. Let's stand up and do what's right before it's too late,' he concluded. Obama was accompanied by Bill Nye 'the Science Guy' on the day trip. He also gave Florida Rep. Debbie Wasserman Schultz, who he's seen here hugging, and Interior Secretary Sally Jewell, second from left, a ride ."
} 

\begin{tabularx}{\textwidth}{@{} l X @{}}
	\toprule
	Autoregressive & President Barack Obama uses Florida Everglades visit to emphasize dangers of climate change. Air Force One uses 9,000 gallons of fuel, round trip, for the trip. White House says Obama's trip will not undermine his Earth Day message.

 \\
	\midrule
SD &     President Barack Obama held up the shrinking Florida Everglades as proof of climate change. Obama said climate change is a problem now, not a problem for another generation. The president said the rising sea level in South Florida and inward flow of salt water is evidence that climate change is already having a negative effect on communities. Obama said the U.S. will lead the world in combatting climate change through the reduction of carbon emissions. Obama's trip to Florida was criticized for the amount of fuel it took to power Air Force One. Obama was accompanied by Bill Nye 'the Science Guy' on the day trip.

 \\
	\midrule
	SD (Lossy) & President Obama visits Everglades National Park today to highlight the threat of climate change. He says the damage is real and immediate and can no longer be denied. Obama says the U.S. will lead the world in reducing carbon emissions. The president's trip to the Everglades is being criticized for the fuel it uses, 9,000 gallons, round trip. Obama says the damage from climate change is not a problem for another generation, but a problem now. The president says the U.S. can solve the problem of climate change and create jobs at the same time.

 \\
	\midrule
SpecCascade & President Obama visits Everglades National Park today to highlight the threat of climate change. Obama says climate change is already having a negative effect on communities in South Florida and that there's not a problem for another generation, 'not anymore.' The president calls the Everglades'magical' and warns that the threat of climate change is not just about the environment, but also about the economies of surrounding towns that rely heavily on tourist dollars. Obama also says he wants to lead the world in combating this threat and that he's committed the United States to reducing carbon emissions.

 \\
	\midrule
	Static Ensemble & President Barack Obama held up the shrinking Florida Everglades as proof positive that climate change is real and is threatening America's 'national treasures'. Obama said climate change can no longer be denied and action can no longer be delayed. The president said the fallout from rising temperatures is 'not a problem for another generation.' At least, 'not anymore,' he said. 'This is a problem now. It has serious implications for the way we live right now.' Obama said the rising sea level in South Florida and inward flow of salt water is evidence that climate change is already having a negative effect on communities.

     \\
	\midrule
	\method &   

       President Barack Obama holds up the shrinking Florida Everglades as proof that climate change is real and is threatening America's 'national treasures.' Obama says climate change can no longer be denied and that action can no longer be delayed. The president visited the subtropical swamps of the park today as part of a push to get Americans thinking and talking about the damage climate change is causing close to home. Obama's trip to Florida was accompanied by Bill Nye 'the Science Guy' and was part of a push to get Americans thinking and talking about the damage climate change is causing close to home.

\\
	\midrule
	\midrule

\end{tabularx}

\newpage

\newpage

\SampleBoxText{Case Study on the CNNDM dataset using Qwen model.}{\small News:     President Barack Obama held up the shrinking Florida Everglades today as proof positive that climate change is real and is threatening America's 'national treasures,' as well as the economies of their surrounding towns, which rely heavily on tourist dollars. 'Climate change can no longer be denied. It can't be edited out,' he said this afternoon in a speech at Everglades National Park commemorating Earth Day. 'It can't be omitted from the conversation. And action can no longer be delayed.' 'That's why I've committed the United States to lead the world in combatting this threat,' he said, throwing a spotlight on his administration's efforts to reduce the rate of global warming through the reduction of carbon emissions. The day trip to South Florida, on which Obama was accompanied by Bill Nye 'the Science Guy,' also highlighted the massive amount of fuel it takes to power Air Force One, however, 9,000 gallons, round trip. Scroll down for video. President Barack Obama takes a walking tour of the Anhinga Trail at Everglades National Park. Obama visited the subtropical swamps of the park today as part of a push to get Americans thinking and talking about the damage climate change is causing close to home . Obama held up the shrinking Florida Everglades today as proof positive that climate change is real and is threatening America's 'national treasures,' as well as the economies of their surrounding towns, which rely heavily on tourist dollars . 'Climate change can no longer be denied. It can't be edited out,' he said this afternoon in a speech at Everglades National Park commemorating Earth Day. 'It can't be omitted from the conversation. And action can no longer be delayed' Vehicles in the president's motorcade are seen in the Everglades National Park where Obama stopped for a visit on the Anhinga trail in Homestead. Asked earlier this week if Obama's gas-guzzling trip to Florida would undermine his Earth Day message, the White House said it would not and reaffirmed Obama's efforts to reduce pollutants . On of the helicopters escorting Marine One with President Barack Obama onboard, prepares to land at Miami's International Airport . I love the smell of jet fuel, Nye told reporters as he waited to board the president's 747 this morning for the flight to Florida, during which the White House said he would film a video with Obama. Asked earlier this week if Obama's gas-guzzling trip to Florida would undermine his Earth Day message, the White House said it would not and reaffirmed Obama's efforts to reduce pollutants. Obama took a half -hour walking tour of the the Anhinga Trail at the 1.5-million-acre national park today before making a plea for Americans to save the Everglades and other natural gems by making environmentally-conscious decisions. 'I can't think of a better way to spend Earth Day than in one of our nation's greatest natural treasures, the Everglades,' he told his audience, calling the swamp, which he acknowledged is not technically a swamp, 'magical.' 'Climate change is threatening this treasure and the communities that depend on it,' Obama warned. 'If we don't act,there may not be an Everglades as we know it.' The president said the fallout from rising temperatures is 'not a problem for another generation.' At least, 'not anymore,' he said. 'This is a problem now. It has serious implications for the way we live right now.' Obama said the rising sea level in South Florida and inward flow of salt water is evidence that climate change is already having a negative effect on communities. 'This is not some impossible problem that we cannot solve. We can solve it if we've got some political will,' he said. 'And we can solve it in a way that creates jobs. We can solve it in a way that doesn't disrupt our economy but enhances our economy. And it's a bipartisan issue.' 'That's why I've committed the United States to lead the world in combatting this threat,' he said, throwing a spotlight on his administration's efforts to reduce the rate of global warming through the reduction of carbon emissions . An alligator is spotted sunning itself along the Anhinga Trail at Everglades National Park as Obama took a tour . Wiping sweat from his brow as he braved the subtropical temperatures to make his mark on Earth Day, the president mingled with crowd-goers at the mid-day event before returning to Washington on Air Force One. Previewing the day trip, the White House had reiterated  Obama's belief that 'no challenge poses a greater threat to future generations than climate change.' It's statement pointed out that last year was the  warmest year on record and that 14 of the 15 of earth's hottest years happened this century. It said, 'Climate change is also affecting some of the most iconic places in our country, from disappearing glaciers in Glacier National Park to dying Joshua Trees in Joshua Tree National Park.' 'Recognized worldwide as a unique and treasured landscape, the Everglades is a perfect example of the threat we face from climate change, including rising sea levels that result in shoreline erosion and increased flooding.' Obama said it is this generation's 'turn to ensure that' natural monuments remain the 'birthright of all Americans for generations to come.' Wiping sweat from his brow as he braved the subtropical temperatures to make his mark on Earth Day, the president mingled with crowd-goers at the mid-day event before returning to Washington on Air Force One . The White House also asserted that climate change was already displacing animals and plants, including tropical orchids, some of which are only found in south Florida.' 'We are blessed with the most beautiful God-given landscape in the world,' the president said today. 'It's an incredible bounty that's been given to us. But we've got to be good stewards for it. We have to take care of it.' Obama said it is this generation's 'turn to ensure that' natural monuments remain the 'birthright of all Americans for generations to come.' 'So many people here are active in your communities, doing what's needed. The young people who are here, the next generation, they're way ahead of us in understanding how important this is . 'Let's make sure we don't disappoint them. Let's stand up and do what's right before it's too late,' he concluded. Obama was accompanied by Bill Nye 'the Science Guy' on the day trip. He also gave Florida Rep. Debbie Wasserman Schultz, who he's seen here hugging, and Interior Secretary Sally Jewell, second from left, a ride ."
} 

\begin{tabularx}{\textwidth}{@{} l X @{}}
	\toprule
	Autoregressive &  1. President Obama highlights shrinking Florida Everglades as proof that climate change is real and a threat to national treasures and local economies . 
    
    2. Obama's trip to Florida on Air Force One used 9,000 gallons of fuel, which he acknowledges, but the White House defends it as part of the effort to reduce emissions . 
    
    3. Obama urges Americans to take environmentally-conscious actions to save the Everglades and other natural gems.

 \\
	\midrule
SD &  1. President Obama highlights Florida Everglades as proof of climate change impact. 

2. Obama criticizes climate change as a current issue, not a future threat. 

3. White House defends gas-guzzling trip to Florida for Earth Day, linking it to efforts to reduce pollutants.

 \\
	\midrule
	SD (Lossy) & 1. President Obama highlights shrinking Florida Everglades as proof climate change is real and threatening national treasures and local economies. 
    
    2. Obama's trip to Florida, which required 9,000 gallons of fuel for Air Force One, drew criticism over its environmental impact. 
    
    3. Obama emphasizes that climate change is a current issue with serious implications, not a problem for future generations. 
    
    4. Obama's administration is committed to leading the world in combating climate change through reducing carbon emissions. 
    
    5. Obama's Earth Day visit to the Everglades was meant to inspire environmentally-conscious decisions to protect natural treasures.

 \\
	\midrule
SpecCascade &   1. President Obama highlights climate change as a threat to the Everglades and national treasures. 

2. The president says climate change is not a problem for another generation. 

3. Obama's trip to Florida was criticized for being gas-guzzling, but the White House said it would not undermine his Earth Day message. 

4. The White House said that climate change is affecting iconic places like the Everglades and Glacier National Park. 

5. Obama emphasizes the need for action to protect natural monuments as a birthright for future generations.

 \\
	\midrule
	Static Ensemble &   1) President Barack Obama visits Florida Everglades as part of Earth Day message. 
    
    2) Obama and Bill Nye 'the Science Guy' on trip. 
    
    3) Obama addresses climate change as a threat to national treasures and economies. 
    
    4) Obama mentions the impact of climate change on the Everglades and surrounding towns. 
    
    5) Obama emphasizes the need to take action on climate change.

     \\
	\midrule
	\method &   

        1. President Obama highlights Florida Everglades as proof of climate change's real and threatening impact. 
        
        2. Obama commits U.S. to lead in combatting climate change and reducing carbon emissions. 
        
        3. Obama's trip to Florida, accompanied by Bill Nye 'the Science Guy,' was seen as a challenge to his Earth Day message.

\\
	\midrule
	\midrule

\end{tabularx}

\newpage

\newpage

\SampleBoxText{Case Study on the CNNDM dataset using Gemma model.}{\small News: President Barack Obama held up the shrinking Florida Everglades today as proof positive that climate change is real and is threatening America's 'national treasures,' as well as the economies of their surrounding towns, which rely heavily on tourist dollars. 'Climate change can no longer be denied. It can't be edited out,' he said this afternoon in a speech at Everglades National Park commemorating Earth Day. 'It can't be omitted from the conversation. And action can no longer be delayed.' 'That's why I've committed the United States to lead the world in combatting this threat,' he said, throwing a spotlight on his administration's efforts to reduce the rate of global warming through the reduction of carbon emissions. The day trip to South Florida, on which Obama was accompanied by Bill Nye 'the Science Guy,' also highlighted the massive amount of fuel it takes to power Air Force One, however, 9,000 gallons, round trip. Scroll down for video. President Barack Obama takes a walking tour of the Anhinga Trail at Everglades National Park. Obama visited the subtropical swamps of the park today as part of a push to get Americans thinking and talking about the damage climate change is causing close to home . Obama held up the shrinking Florida Everglades today as proof positive that climate change is real and is threatening America's 'national treasures,' as well as the economies of their surrounding towns, which rely heavily on tourist dollars . 'Climate change can no longer be denied. It can't be edited out,' he said this afternoon in a speech at Everglades National Park commemorating Earth Day. 'It can't be omitted from the conversation. And action can no longer be delayed' Vehicles in the president's motorcade are seen in the Everglades National Park where Obama stopped for a visit on the Anhinga trail in Homestead. Asked earlier this week if Obama's gas-guzzling trip to Florida would undermine his Earth Day message, the White House said it would not and reaffirmed Obama's efforts to reduce pollutants . On of the helicopters escorting Marine One with President Barack Obama onboard, prepares to land at Miami's International Airport . I love the smell of jet fuel, Nye told reporters as he waited to board the president's 747 this morning for the flight to Florida, during which the White House said he would film a video with Obama. Asked earlier this week if Obama's gas-guzzling trip to Florida would undermine his Earth Day message, the White House said it would not and reaffirmed Obama's efforts to reduce pollutants. Obama took a half -hour walking tour of the the Anhinga Trail at the 1.5-million-acre national park today before making a plea for Americans to save the Everglades and other natural gems by making environmentally-conscious decisions. 'I can't think of a better way to spend Earth Day than in one of our nation's greatest natural treasures, the Everglades,' he told his audience, calling the swamp, which he acknowledged is not technically a swamp, 'magical.' 'Climate change is threatening this treasure and the communities that depend on it,' Obama warned. 'If we don't act,there may not be an Everglades as we know it.' The president said the fallout from rising temperatures is 'not a problem for another generation.' At least, 'not anymore,' he said. 'This is a problem now. It has serious implications for the way we live right now.' Obama said the rising sea level in South Florida and inward flow of salt water is evidence that climate change is already having a negative effect on communities. 'This is not some impossible problem that we cannot solve. We can solve it if we've got some political will,' he said. 'And we can solve it in a way that creates jobs. We can solve it in a way that doesn't disrupt our economy but enhances our economy. And it's a bipartisan issue.' 'That's why I've committed the United States to lead the world in combatting this threat,' he said, throwing a spotlight on his administration's efforts to reduce the rate of global warming through the reduction of carbon emissions . An alligator is spotted sunning itself along the Anhinga Trail at Everglades National Park as Obama took a tour . Wiping sweat from his brow as he braved the subtropical temperatures to make his mark on Earth Day, the president mingled with crowd-goers at the mid-day event before returning to Washington on Air Force One. Previewing the day trip, the White House had reiterated  Obama's belief that 'no challenge poses a greater threat to future generations than climate change.' It's statement pointed out that last year was the  warmest year on record and that 14 of the 15 of earth's hottest years happened this century. It said, 'Climate change is also affecting some of the most iconic places in our country, from disappearing glaciers in Glacier National Park to dying Joshua Trees in Joshua Tree National Park.' 'Recognized worldwide as a unique and treasured landscape, the Everglades is a perfect example of the threat we face from climate change, including rising sea levels that result in shoreline erosion and increased flooding.' Obama said it is this generation's 'turn to ensure that' natural monuments remain the 'birthright of all Americans for generations to come.' Wiping sweat from his brow as he braved the subtropical temperatures to make his mark on Earth Day, the president mingled with crowd-goers at the mid-day event before returning to Washington on Air Force One . The White House also asserted that climate change was already displacing animals and plants, including tropical orchids, some of which are only found in south Florida.' 'We are blessed with the most beautiful God-given landscape in the world,' the president said today. 'It's an incredible bounty that's been given to us. But we've got to be good stewards for it. We have to take care of it.' Obama said it is this generation's 'turn to ensure that' natural monuments remain the 'birthright of all Americans for generations to come.' 'So many people here are active in your communities, doing what's needed. The young people who are here, the next generation, they're way ahead of us in understanding how important this is . 'Let's make sure we don't disappoint them. Let's stand up and do what's right before it's too late,' he concluded. Obama was accompanied by Bill Nye 'the Science Guy' on the day trip. He also gave Florida Rep. Debbie Wasserman Schultz, who he's seen here hugging, and Interior Secretary Sally Jewell, second from left, a ride ."
} 

\begin{tabularx}{\textwidth}{@{} l X @{}}
	\toprule
	Autoregressive & President Obama held up the shrinking Florida Everglades today as proof positive that climate change is real and is threatening America's 'national treasures'. 
    
    Obama says climate change is already having a negative effect on communities. 
    
    The president warned that if action isn't taken, there may not be an Everglades as we know it. 
    He was accompanied by Bill Nye \"the Science Guy\" on the day trip.

 \\
	\midrule
SD &  9,000 gallons of fuel used for Air Force One round trip . Obama held up the shrinking Florida Everglades as proof positive that climate change is real . Obama said it is this generation's 'turn to ensure that' natural monuments remain the 'birthright of all Americans for generations to come.' .

 \\
	\midrule
	SD (Lossy) &  9,000 gallons of fuel used for Air Force One round trip to Florida . Obama held up the shrinking Florida Everglades as proof positive that climate change is real . President said the rising sea level in South Florida and inward flow of salt water is evidence that climate change is already having a negative effect on communities . Obama said it is this generation's 'turn to ensure that' natural monuments remain the 'birthright of all Americans for generations to come.'.

 \\
	\midrule
SpecCascade &   9,000 gallons of fuel used for Air Force One round trip . Obama held up the shrinking Florida Everglades as proof positive that climate change is real . Obama said it is this generation's 'turn to ensure that' natural monuments remain the 'birthright of all Americans for generations to come.'.

 \\
	\midrule
	Static Ensemble &   1. President Obama held up the shrinking Florida Everglades as proof of climate change.
    
    2. Obama's trip to Florida included a walking tour of the Everglades and a speech on Earth Day. 
    
    3. The trip highlighted the administration's efforts to reduce carbon emissions. 
    
    4. Obama's gas-guzzling trip to Florida was criticized, but the White House defended it. 
    
    5. Obama emphasized the urgency of addressing climate change and its impact on communities."

     \\
	\midrule
	\method &   

        1. Obama uses Everglades as proof of climate change's impact. 
        
        2. Trip highlights shrinking Everglades and economic impact on surrounding towns. 
        
        3. Obama's Air Force One trip uses 9,000 gallons of fuel, raising questions about hypocrisy. 
        
        4. Bill Nye 'the Science Guy' accompanies Obama. 
        
        5. Obama emphasizes the urgency of addressing climate change and its impact on future generations."

\\
	\midrule
	\midrule

\end{tabularx}

\newpage

\newpage

\subsection{XSum Case Study}

\SampleBoxText{Case Study on the XSum dataset using Llama model.}{\footnotesize Gao Yu, 71, was found guilty last April and challenged her conviction at a closed hearing in Beijing on Thursday.The guilty verdict was upheld, but the jail term was cut from seven years to five and the medical grounds permitted.Foreign governments and rights groups have condemned the guilty verdict, calling it politically motivated.Sources have told the BBC that Ms Gao has already left prison.The well-known investigative journalist was found guilty of sending an internal Communist Party document to Mingjing News, a Chinese-language website in the United States. The document warned against the dangers of a free press and an independent civil society.At the time of the original conviction the court had said that Ms Gao had \"illegally provided state secrets to foreigners\".Both Ms Gao and Mingjing News had denied she was the source of the leak.However, on Thursday, Ms Gao confessed to the crime and expressed regret for her actions, and the appeals court reduced her sentence to five years and decided Ms Gao could \"temporarily serve the sentence out of prison\" as she was \"severely ill\", Chinese state media said.Experts say Chinese courts often grant reduced sentences after convicts plead guilty. However, Ms Gao's lawyers told AP they continued to argue that she was innocent.Many believe that Gao Yu is unpopular with the Chinese government because of her unrelenting reports focusing on China's elite-level politics, the BBC's Celia Hatton in Beijing reports.Many people had pressed for a lighter sentence in consideration of Ms Gao's advanced age and her chronic heart problems.Her son, Zhao Meng, had told the BBC he was concerned his mother would not survive another lengthy period in prison, her third since 1989.Prominent human rights activist Hu Jia said (in Chinese): \"I'm relieved Gao Yu can return to her own home, after a year and a half in custody... but the risk of being returned to jail at any time still hangs over Gao Yu's head, like the sword of Damocles.
} 

\begin{tabularx}{\textwidth}{@{} l X @{}}
	\toprule
	Autoregressive &   Chinese journalist Gao Yu has had her seven-year prison sentence cut to five years and allowed to serve at home. 
    Her sentence was reduced after she confessed to the crime and expressed regret.

 \\
	\midrule
SD &   A Chinese journalist's conviction for leaking state secrets is upheld, but her jail term is reduced to five years.  She is released from prison due to health reasons.  She confessed to the crime.  She is still at risk of being returned to jail.  She is a prominent investigative journalist.  She was found guilty of sending an internal Communist Party document to a Chinese-language website in the United States.

 \\
	\midrule
	SD (Lossy) &  A Chinese journalist's conviction for leaking state secrets is upheld, but her jail term is reduced to five years.  She is released from prison due to health reasons.  She confessed to the crime.  She is still at risk of being returned to jail.  She is a prominent investigative journalist.  She was found guilty of sending an internal Communist Party document to a Chinese-language website in the United States.

 \\
	\midrule
SpecCascade &  A Chinese journalist's conviction for leaking state secrets is upheld, but her jail term is reduced to five years.  She is released from prison due to health reasons.  She confessed to the crime.  She is still at risk of being returned to jail.  She is a prominent investigative journalist.  She was found guilty of sending an internal Communist Party document to a Chinese-language website in the United States.  She was found guilty of illegally providing state secrets to foreigners.

 \\
	\midrule
	Static Ensemble &  71-year-old Chinese investigative journalist Gao Yu's conviction is upheld, but her sentence is reduced from 7 years to 5 years. Gao Yu, a 71-year-old Chinese journalist, is released from prison after her conviction is upheld. Gao Yu's conviction for leaking state secrets is upheld but her sentence is reduced to 5 years. Chinese authorities reduce Gao Yu's 7-year sentence to 5 years, allowing her to leave prison early due to her age and health. Gao Yu's sentence is reduced to 5 years, allowing her to return home after serving a year and a half in custody.

     \\
	\midrule
	\method &   71-year-old Chinese journalist Gao Yu has been released from prison after her conviction was upheld, but her sentence was reduced to five years. She was found guilty of leaking state secrets.  Gao Yu has been released from prison after her conviction was upheld, but her sentence was reduced to five years.  She was found guilty of leaking state secrets.  Gao Yu has been released from prison after her conviction was upheld, but her sentence was reduced to five years.

\\
	\midrule
	\midrule

\end{tabularx}

\newpage

\newpage

\SampleBoxText{Case Study on the XSum dataset using Gemma model.}{Gao Yu, 71, was found guilty last April and challenged her conviction at a closed hearing in Beijing on Thursday.The guilty verdict was upheld, but the jail term was cut from seven years to five and the medical grounds permitted.Foreign governments and rights groups have condemned the guilty verdict, calling it politically motivated.Sources have told the BBC that Ms Gao has already left prison.The well-known investigative journalist was found guilty of sending an internal Communist Party document to Mingjing News, a Chinese-language website in the United States. The document warned against the dangers of a free press and an independent civil society.At the time of the original conviction the court had said that Ms Gao had \"illegally provided state secrets to foreigners\".Both Ms Gao and Mingjing News had denied she was the source of the leak.However, on Thursday, Ms Gao confessed to the crime and expressed regret for her actions, and the appeals court reduced her sentence to five years and decided Ms Gao could \"temporarily serve the sentence out of prison\" as she was \"severely ill\", Chinese state media said.Experts say Chinese courts often grant reduced sentences after convicts plead guilty. However, Ms Gao's lawyers told AP they continued to argue that she was innocent.Many believe that Gao Yu is unpopular with the Chinese government because of her unrelenting reports focusing on China's elite-level politics, the BBC's Celia Hatton in Beijing reports.Many people had pressed for a lighter sentence in consideration of Ms Gao's advanced age and her chronic heart problems.Her son, Zhao Meng, had told the BBC he was concerned his mother would not survive another lengthy period in prison, her third since 1989.Prominent human rights activist Hu Jia said (in Chinese): \"I'm relieved Gao Yu can return to her own home, after a year and a half in custody... but the risk of being returned to jail at any time still hangs over Gao Yu's head, like the sword of Damocles.
} 

\begin{tabularx}{\textwidth}{@{} l X @{}}
	\toprule
	Autoregressive & A Chinese journalist, Gao Yu, had her conviction upheld, jail term reduced, and released from prison \textcolor{blue}{on medical grounds.}

 \\
	\midrule
SD &  A Chinese journalist Gao Yu's conviction was upheld, sentence reduced, and she has been released from prison \textcolor{blue}{on medical grounds.}

 \\
	\midrule
	SD (Lossy) &  A Chinese journalist's conviction was upheld, her sentence reduced, and she has been released from prison on \textcolor{blue}{medical grounds.}

 \\
	\midrule
SpecCascade &  A Chinese journalist's conviction was upheld, but her sentence was reduced and she was released due to illness and a guilty plea.

 \\
	\midrule
	Static Ensemble & Chinese journalist Gao Yu's conviction was upheld, sentence reduced, and she has been released from prison \textcolor{green}{due to health concerns.}

     \\
	\midrule
	\method &   Gao Yu, a journalist, was convicted of leaking state secrets, but her sentence was reduced and she has been released from prison.

\\
	\midrule
	\midrule

\end{tabularx}

\newpage

\newpage

\SampleBoxText{Case Study on the XSum dataset using Qwen model.}{\small Gao Yu, 71, was found guilty last April and challenged her conviction at a closed hearing in Beijing on Thursday.The guilty verdict was upheld, but the jail term was cut from seven years to five and the medical grounds permitted.Foreign governments and rights groups have condemned the guilty verdict, calling it politically motivated.Sources have told the BBC that Ms Gao has already left prison.The well-known investigative journalist was found guilty of sending an internal Communist Party document to Mingjing News, a Chinese-language website in the United States. The document warned against the dangers of a free press and an independent civil society.At the time of the original conviction the court had said that Ms Gao had \"illegally provided state secrets to foreigners\".Both Ms Gao and Mingjing News had denied she was the source of the leak.However, on Thursday, Ms Gao confessed to the crime and expressed regret for her actions, and the appeals court reduced her sentence to five years and decided Ms Gao could \"temporarily serve the sentence out of prison\" as she was \"severely ill\", Chinese state media said.Experts say Chinese courts often grant reduced sentences after convicts plead guilty. However, Ms Gao's lawyers told AP they continued to argue that she was innocent.Many believe that Gao Yu is unpopular with the Chinese government because of her unrelenting reports focusing on China's elite-level politics, the BBC's Celia Hatton in Beijing reports.Many people had pressed for a lighter sentence in consideration of Ms Gao's advanced age and her chronic heart problems.Her son, Zhao Meng, had told the BBC he was concerned his mother would not survive another lengthy period in prison, her third since 1989.Prominent human rights activist Hu Jia said (in Chinese): \"I'm relieved Gao Yu can return to her own home, after a year and a half in custody... but the risk of being returned to jail at any time still hangs over Gao Yu's head, like the sword of Damocles.
} 

\begin{tabularx}{\textwidth}{@{} l X @{}}
	\toprule
	Autoregressive &  71-year-old Chinese journalist Gao Yu was sentenced to five years in prison for leaking a Party document but was allowed to serve it outside prison due to illness.

 \\
	\midrule
SD &   \textcolor{blue}{71-year-old Chinese journalist Gao Yu was convicted of leaking state secrets, received a five-year sentence, and is temporarily released due to illness.}

 \\
	\midrule
	SD (Lossy) &   \textcolor{blue}{71-year-old Chinese journalist Gao Yu was convicted of leaking state secrets, received a reduced sentence, and is now out of prison due to illness.}

 \\
	\midrule
SpecCascade &  \textcolor{blue}{71-year-old Chinese journalist Gao Yu was convicted of leaking state secrets, received a reduced sentence, and is now out of prison due to illness.}

 \\
	\midrule
	Static Ensemble &  71-year-old Chinese journalist Gao Yu was found guilty but received a reduced sentence, and the court allowed her to temporarily serve her sentence out of prison due to her serious illness.

     \\
	\midrule
	\method &    71-year-old Chinese journalist Gao Yu was convicted of leaking state secrets, received a five-year sentence, and is now allowed to leave prison temporarily due to illness, despite her lawyers' claims of innocence.

\\
	\midrule
	\midrule

\end{tabularx}

\newpage

\newpage

\subsection{MBPP Case Study}

\SampleBoxText{Case Study on the MBPP dataset using Llama model.}{Write a function to find the median of three numbers.
}

\begin{tabularx}{\textwidth}{@{} l X @{}}
	\toprule

	Autoregressive &   
    \textcolor{green}{\codecell{
def median_numbers(a,b,c):
    numbers = sorted([a,b,c])
    if len(numbers) %
        return (numbers[len(numbers)//2 - 1] + numbers[len(numbers)//2]) / 2
    else:
        return numbers[len(numbers)//2]
}}

\\
	\midrule
SD &   \textcolor{green}{\codecell{
def median_numbers(a,b,c):
    numbers = sorted([a,b,c])
    if len(numbers) %
        return (numbers[len(numbers)//2 - 1] + numbers[len(numbers)//2]) / 2
    else:
        return numbers[len(numbers)//2]
}}
 \\
	\midrule
	SD (Lossy) & \textcolor{blue}{\codecell{
def median_numbers(a,b,c):
    return sorted([a,b,c])[1]
}}

 \\
	\midrule
SpecCascade &   \textcolor{green}{\codecell{
def median_numbers(a,b,c):
    numbers = sorted([a,b,c])
    if len(numbers) %
        median = (numbers[len(numbers)//2 - 1] + numbers[len(numbers)//2]) / 2
    else:
        median = numbers[len(numbers)//2]
    return median
}}

 \\
	\midrule
	Static Ensemble &   \textcolor{blue}{\codecell{
def median_numbers(a,b,c):
    return sorted([a,b,c])[1]
}}
    
     \\
	\midrule
	\method &   

\textcolor{blue}{\codecell{
def median_numbers(a,b,c):
    return sorted([a,b,c])[1]
}}

\\
	\midrule
	\midrule

\end{tabularx}

\newpage

\SampleBoxText{Case Study on the MBPP dataset using Qwen model.}{Write a function to find the median of three numbers.
}

\begin{tabularx}{\textwidth}{@{} l X @{}}
	\toprule

	Autoregressive &   
    \textcolor{blue}{\codecell{
def median_numbers(a,b,c):
    return sorted([a,b,c])[1]
}}

\\
	\midrule
SD &   \textcolor{blue}{\codecell{
def median_numbers(a,b,c):
    return sorted([a,b,c])[1]
}}

 \\
	\midrule
	SD (Lossy) & \textcolor{blue}{\codecell{
def median_numbers(a,b,c):
    return sorted([a,b,c])[1]
}}

 \\
	\midrule
SpecCascade &   \textcolor{blue}{\codecell{
def median_numbers(a,b,c):
    return sorted([a,b,c])[1]
}}

 \\
	\midrule
	Static Ensemble &   \textcolor{blue}{\codecell{
def median_numbers(a,b,c):
    return sorted([a,b,c])[1]
}}
    
     \\
	\midrule
	\method &   

\textcolor{blue}{\codecell{
def median_numbers(a,b,c):
    return sorted([a,b,c])[1]
}
}
\\
	\midrule
	\midrule
\end{tabularx}

\newpage

\SampleBoxText{Case Study on the MBPP dataset using Gemma model.}{Write a function to find the median of three numbers.
}

\begin{tabularx}{\textwidth}{@{} l X @{}}
	\toprule

	Autoregressive &   
    \textcolor{blue}{\codecell{
def median_numbers(a,b,c):
    if (a <= b and b <= c) or (c <= b and b <= a):
        return b
    elif (b <= a and a <= c) or (c <= a and a <= b):
        return a
    else:
        return c
}}
\\
	\midrule
SD &   \codecell{
def median_numbers(a,b,c):
    numbers = sorted([a,b,c])
    if len(numbers) %
        return (numbers[len(numbers)//2 - 1] + numbers[len(numbers)//2]) / 2
    else:
        return numbers[len(numbers)//2]
}
 \\
	\midrule
	SD (Lossy) & \textcolor{blue}{\codecell{
def median_numbers(a,b,c):
    if (a <= b and b <= c) or (c <= b and b <= a):
        return b
    elif (b <= a and a <= c) or (c <= a and a <= b):
        return a
    else:
        return c
}
}

 \\
	\midrule
SpecCascade &  \textcolor{blue}{\codecell{
def median_numbers(a,b,c):
    if (a <= b and b <= c) or (c <= b and b <= a):
        return b
    elif (b <= a and a <= c) or (c <= a and a <= b):
        return a
    else:
        return c
}
}
 \\
	\midrule
	Static Ensemble &   \textcolor{blue}{\codecell{
def median_numbers(a,b,c):
    if (a <= b and b <= c) or (c <= b and b <= a):
        return b
    elif (b <= a and a <= c) or (c <= a and a <= b):
        return a
    else:
        return c
}
}
    
     \\
	\midrule
	\method &   \textcolor{blue}{\codecell{
def median_numbers(a,b,c):
    if (a <= b and b <= c) or (c <= b and b <= a):
        return b
    elif (b <= a and a <= c) or (c <= a and a <= b):
        return a
    else:
        return c
}
}

\\
	\midrule
	\midrule

\end{tabularx}

\end{document}